\documentclass{article}

\usepackage[preprint]{neurips_2020_arxivadapted}
\usepackage[utf8]{inputenc} % allow utf-8 input
\usepackage[T1]{fontenc}    % use 8-bit T1 fonts
\usepackage{hyperref}       % hyperlinks
\usepackage{url}            % simple URL typesetting
\usepackage{booktabs}       % professional-quality tables
\usepackage{amsfonts}       % blackboard math symbols
\usepackage{nicefrac}       % compact symbols for 1/2, etc.
\usepackage{microtype}      % microtypography

\usepackage{float}
\usepackage{graphicx}

\usepackage{color}
\usepackage{amsmath}
\usepackage{amssymb}
\usepackage{amsthm}
\usepackage{mathtools}
\usepackage[mathscr]{eucal}% \mathscr{ABCDEFGHIJKLMNOPQRSTUVWXYZ}
\usepackage{bm}% bold math
\usepackage{bbm}
\usepackage{relsize}
\usepackage{graphics}
\usepackage{wrapfig}
\usepackage{multirow}
\usepackage{enumitem} 

\usepackage{array}

\usepackage{pgf}
\usepackage{xinttools}%,pgfkeys,pgfmath}
\usepackage{tikz}
\usetikzlibrary{arrows.meta}
\usepackage{textcomp}
\usetikzlibrary{calc,shapes,arrows,positioning,automata,trees}
\usepackage{placeins} % Provides \FloatBarrier
\usepackage{tabularx}
\usepackage{graphicx}
\usepackage{subcaption} % subfigures
\usepackage{listings}
\usepackage[toc,page]{appendix}

\newtheorem{theorem}{Theorem}

\newtheorem{corollary}{Corollary}%
\newtheorem{lemma}{Lemma}%
\newtheorem*{theorem*}{Theorem}
\newtheorem*{definition*}{Definition}

\newcommand\Ba{\bm{a}}

\newcommand\Bm{\bm{m}}

\newcommand\Bp{\bm{p}}
\newcommand\Bq{\bm{q}}

\newcommand\Bw{\bm{w}}

\newcommand\By{\bm{y}}
\newcommand\Bz{\bm{z}}
\newcommand\Bla{\bm{\lambda}}

\newcommand\Bth{\bm{\theta}}

\newcommand\Bps{\bm{\psi}}

\newcommand\Bom{\bm{\omega}}
 \newcommand{\dR}{\mathbb{R}}

\newcommand{\rE}{\mathrm{E}}

\newcommand{\rK}{\mathrm{K}} \newcommand{\rL}{\mathrm{L}}
 
 \newcommand{\rP}{\mathrm{P}}

 \newcommand{\cP}{\mathcal{P}}

\newcommand{\sA}{\mathscr{A}}

 \newcommand{\sR}{\mathscr{R}}
\newcommand{\sS}{\mathscr{S}}

 \newcommand{\Rd}{\mathrm{d}}

\newcommand{\Ro}{\mathrm{o}}

\newcommand\EXP{\mathbf{\mathrm{E}}}

\newcommand\argmax{\mathop{\mathrm{argmax}\,}}

\newcommand{\ABS}[1]{{{\left| #1 \right|}}} % |1|
\renewcommand{\leq}{\leqslant}
\renewcommand{\geq}{\geqslant}

\usepackage{pdflscape} % For landscape tables

\newcolumntype{R}[1]{>{\raggedright\arraybackslash}p{#1}}
\newcolumntype{C}[1]{>{\centering\arraybackslash}p{#1}}
\newcolumntype{L}[1]{>{\raggedleft\arraybackslash}p{#1}}

\usepackage{bigdelim}
\usepackage{colortbl} % colored cells
\definecolor{mColor1}{rgb}{0.95,0.95,0.95}

\hypersetup{
   colorlinks=true,
   linkcolor=blue,
   citecolor=green,
   urlcolor=magenta,
   pdfborder=0 0 0,
   pdftitle={Convergence Proof for Actor-Critic Methods Applied to PPO and RUDDER},
   pdfsubject={Reinforcement Learning, Actor-Critic Methods, Stochastic Approximation, Convergence, 
               PPO, RUDDER}, 
   pdfkeywords={Reinforcement Learning, Actor-Critic Methods, Stochastic Approximation, Convergence, 
               PPO, RUDDER, Two-timescale Approximation},
   pdfauthor={Markus Holzleitner, Lukas Gruber, Jos\'{e} Arjona-Medina, Johannes Brandstetter, Sepp Hochreiter},%
   pdfstartview=FitH
}

\title{Convergence Proof for Actor-Critic Methods
Applied to PPO and RUDDER}

\author{\vspace{0.1cm}
     Markus Holzleitner\footnotemark[1] 
     \qquad \bf Lukas Gruber\footnotemark[1] 
     \qquad \bf Jos\'{e} Arjona-Medina\footnotemark[1]
    \\\ \qquad  \bf Johannes Brandstetter\footnotemark[1] \qquad 
     \bf Sepp Hochreiter\footnotemark[1]~$~^{,}$\footnotemark[2] \\ 
  \footnotemark[1]~~ELLIS Unit Linz and LIT AI Lab, \\ 
                    Institute for Machine Learning,\\
                    Johannes Kepler University Linz, Austria \\
  \footnotemark[2]~~Institute of Advanced Research in 
                    Artificial Intelligence (IARAI)\\
 }

\DeclareUnicodeCharacter{2212}{-}

\begin{document}

\maketitle

\begin{abstract}
We prove under commonly used assumptions the convergence of 
actor-critic reinforcement learning algorithms, which
simultaneously learn a policy function, the actor, 
and a value function, the critic.
Both functions can be deep neural networks of arbitrary complexity.
Our framework allows showing convergence of
the well known Proximal Policy Optimization (PPO) 
and of the recently introduced RUDDER.
For the convergence proof  
we employ recently introduced 
techniques from the two time-scale stochastic approximation theory.
Our results are valid for 
actor-critic methods that use episodic samples
and that have a policy that becomes more greedy during learning.
Previous convergence proofs 
assume linear function approximation,
cannot treat episodic examples, or
do not consider that policies become greedy.
The latter is relevant since optimal policies
are typically deterministic. 
\end{abstract}

\section{Introduction}

In reinforcement learning, 
popular methods like
Proximal Policy Optimization (PPO) \citep{Schulman:17}
lack convergence proofs. 
Convergence proofs for these methods are challenging, 
since they use 
deep neural networks, 
episodes as samples, 
policies that become greedy, and
previous policies for trust region methods.
For $Q$-learning, convergence to an optimal policy has
been proven in \citet{Watkins:92,Bertsekas:96} 
as well as for TD($\lambda$) in \citet{Dayan:92}.
Convergence of SARSA to an optimal policy has been 
established for policies that become greedy, like
``greedy in the limit with infinite exploration'' (GLIE) or
``restricted rank-based randomized'' (RRR) \citep{Singh:00}.
Policy gradient methods converge to a local optimum, since
the ``policy gradient theorem'' \citep[Chapter~13.2]{Sutton:18book}
shows that they form a stochastic gradient of the objective.
Stochastic gradients converge according to
the stochastic approximation theory to an optimum \citep{Robbins:51,Kushner:78,Tsitsiklis:94,Borkar:00,Kushner:03,Borkar:08,Bhatnagar:13}.
Temporal difference (TD) convergences to a local optimum with 
smooth function approximation like by neural networks \citep{Maei:09}.
Also Deep $Q$-Networks (DQNs) \citep{Mnih:13,Mnih:15} use
a single neural network, therefore can be shown to 
converge, as done in \citet{Fan:20}. However it is assumed that
every training set of reward-state transitions
is drawn iid and that 
a global minimum of the $Q$-function on the training set 
is provided.

We prove the convergence of 
general actor-critic reinforcement 
learning algorithms \citep[Chapter~13.5]{Sutton:18book}.
Recently, actor-critic methods have had 
a considerable success, 
e.g.\ at defeating humans in the game Dota~2 \citep{Berner:19}
and in mastering the game of Starcraft~II \citep{Vinyals:19}.
%Therefore, actor-critic methods are currently very popular
%reinforcement learning methods.
%Our convergence proof for actor-critic methods
%focuses on convergence but not on 
%convergence to an optimal policy,
%since that would require to ensure that neural networks do not get stuck in 
%local minima.
Actor-critic algorithms simultaneously learn a 
policy function, the actor, 
and a critic function that estimates  
values, action-values, advantages, or redistributed rewards.
The critic is responsible for credit assignment, that is,
which action or 
state-action pair was responsible for receiving a reward.
Using this credit assignment, a policy function is updated to increase the 
return.
Actor-critic algorithms are typically policy gradient
methods, but can also be reward redistribution methods 
like RUDDER \citep{Arjona-Medina:19} 
or ``backpropagation through a model'' 
\citep{Munro:87,Robinson:89,RobinsonFallside:89,Bakker:07}.
Actor-critic algorithms 
have been only proven to converge for simple settings 
like for the neural networks that are linear \citep{Konda:99NIPS,Konda:99,Xu:19,Yang:19,Liu:19}.
In contrast to these convergence proofs,
in our setting both functions 
can be deep neural networks of arbitrary complexity,
though they should not share weights. 

The main contribution of this paper is to provide 
a convergence proof for general actor-critic reinforcement 
learning algorithms.
We apply this convergence proof
to two concrete actor-critic methods. 
First, we establish convergence of
a practical variant of 
Proximal Policy Optimization (PPO) \citep{Schulman:17}.
PPO is an actor-critic on-policy gradient method
with trust region penalties to ensure
a small policy gap \citep{Schulman:15tr}.
Secondly, we prove convergence of the 
recently introduced RUDDER \citep{Arjona-Medina:19}.
RUDDER targets the problem of sparse and delayed rewards
by reward redistribution which 
directly and efficiently assigns reward 
to relevant state-action pairs. 
Thus, RUDDER dramatically speeds up learning for sparse and delayed rewards.
In RUDDER, the critic is the reward redistributing network, which is typically 
an LSTM.

The main proof techniques are recent developments from 
the two time-scale stochastic approximation theory \citep{Borkar:08}. 
The recent addition to the theory is the introduction of
controlled Markov processes \citep{Karmakar:17}, 
which can treat policies that become more greedy and 
trust region methods that use previous policies. 
The two time-scale stochastic approximation framework has been
applied previously to show convergence of actor-critic algorithms 
\citep{Konda:99NIPS,Konda:99} and, more recently, 
of Linear Quadratic Regulator (LQR) problems \citep{Xu:19} 
and off-policy TD learning \citep{Yang:19}. 
However, only tabular cases or linear function approximations
have been considered.
In a recent work, convergence was  
shown for variants of 
PPO and Trust Region Policy Optimization (TRPO) equipped 
with neural networks \citep{Liu:19}.
However, again the neural networks were only linear, 
the policy was energy-based, and 
the Kullback-Leibler term of the trust-region method
was modified.

We aim at generalizing these proofs to learning settings
which use deep neural networks,
use episodes, use policies that become greedy, and use trust region methods. 
Therefore, the idea of stationary distributions 
on state-action pairs does not apply \citep{Konda:99NIPS,Konda:99}
and we have to enrich the framework by a controlled Markov process which describes 
how the policy becomes more greedy and how to use previous policies.
%MH: maybe add comments on drawbacks resulting from assumptions
While we are developing a framework to ensure convergence, it does not
imply convergence to an optimal policy. 
Such proofs are in general difficult for methods that
use deep neural networks, since locally stable attractors 
may not correspond to optimal policies \citep{Mazumdar:19,Jin:19,Lin:19}.
However, convergence to a locally optimal policy can be proven 
for linear approximation
to $Q$-values \citep{Sutton:00,Konda:03}. 
%MH. Todo (was also remarked by reviewers): provide literature which fails at these tasks
Our main contributions to the convergence proof are, that we:
\begin{itemize}[topsep=-5pt,itemsep=1pt,partopsep=0pt, parsep=0pt,wide, labelwidth=!, labelindent=0pt]
\item use a Markov control in the
two time-scale stochastic approximation framework,
\item use episodes as samples instead of transitions,
\item allow policies to become greedy,
\item allow objectives that use previous policies (trust region methods). 
\end{itemize}

In the next section,
the main theorem is provided, which
shows local convergence of actor-critic methods. 
Next, we formulate the results for PPO and RUDDER as corollaries.
The third section gives a roadmap for the corresponding proofs, 
thereby introducing the precise framework and the results 
from stochastic approximation theory \citep{Borkar:08,Karmakar:17}. 
Finally, we discuss the technical assumptions and details for the proofs.

\section{The Main Results}

\subsection{Abstract setting and Main Theorem}
\paragraph{Preliminaries}
We consider a finite MDP defined by the 4-tuple 
$\cP=(\sS,\sA,\sR,p)$ (we assume a discount factor $\gamma=1$)
where the state space $\sS$ and 
the action space $\sA$ consist of
finitely many states $s$ and actions $a$
and $\sR$ the set of rewards $r$ which are bounded.
Let us denote by $|\sA|$ and $|\sS|$ the corresponding cardinalities and $K_R>0$ an upper bound on the absolute values of the rewards.
For a given time step $t$, the random variables for state, action, and reward are 
$S_t$, $A_t$ and $R_{t+1}=R(S_t,A_t)$, respectively.
Furthermore, $\cP$ has
transition-reward distributions
$p(S_{t+1}=s',R_{t+1}=r \mid S_t=s,A_t=a)$.
By $\pi$ we denote an associated Markov policy.
%MH: The word "measurable" was heavily critizised by reviewers. This is indeed ambiguous at this point.
%A Markov policy $\pi$ is a measurable mapping
%$\pi: \cS \rightarrow \cA$.
The (undiscounted) return of a sequence of length $T$ at time $t$ is 
$G_t = \sum_{k=0}^{T-t}  R_{t+k+1}$.
As usual, the action-value function for a given policy $\pi$ is
$q^{\pi}(s,a) = \EXP_{\pi} \left[ G_t \mid S_t=s, A_t=a \right]$.
The goal is to find the optimal policy
$\pi^* = \argmax_\pi \EXP_\pi [G_0]$. We assume that the states $s$ are time-aware 
(time $t$ can be extracted from each state)
in order to guarantee stationary optimal policies.

The abstract actor-critic setting 
is assumed to have two loss functions: $\rL_h$ for the policy 
and $\rL_g$ for the critic. Additionally we have the following building blocks:
\begin{itemize}[topsep=-5pt,itemsep=1pt,partopsep=0pt, parsep=0pt,wide, labelwidth=!, labelindent=0pt]
\item We consider two classes of parameters, denoted by $\Bom \in \dR^m $ and $\Bth \in \dR^k$. Moreover, $\Bz$ denotes an additional controlled Markov process with values in a compact metric space that may allow e.g. to force the policy to get more greedy and for treating
trust region methods which rely on previous policies (it may be used for other purposes as well, e.g. Markovian sampling). 
%$\Bz$ will take its values in a compact metric space and have some sort of a "unique limit".
$\Bz$ will be defined in a similar abstract way as done in \citet{Karmakar:17} to make the setting as general as possible. We defer the technical details to Section \ref{borkarsetting}.
\item The first loss $\rL_h(\Bth, \Bom, \Bz)$ is minimized 
with respect to $\Bth$ in order to find an optimal policy. This is achieved by updating a sufficiently smooth policy $\pi( \Bth,\Bz)$, that can be controlled by $\Bz$. We will discuss in Section~\ref{assumptions}, how $\pi$ can be constructed in specific situations.
Next we consider two optional possibilities, how $\rL_h(\Bth, \Bom, \Bz)$ may be defined: it may equal the expectation (i) $\EXP_{\tau \sim \pi(\Bth,\Bz) } 
    \left[ \phi(\tau,\Bth, \Bom,\Bz) \right]$
    or (ii) $\EXP_{\tau \sim  \breve{\pi}} 
    \left[ \phi(\pi(.;\Bth, \Bz),\tau,\Bth, \Bom,\Bz) \right]$ where the expectations are taken over whole episodes $\tau=(s_0,a_0,\ldots,s_T,a_T)$ (sequences) 
that are generated via (i) $\pi(\Bth, \Bz)$ or 
(ii) a behavioral policy $ \breve{\pi}$, respectively. It will be clear from the context, which of these two possibilities we are using.
The function $\phi$ can be interpreted as a per-sample loss for a sufficiently smooth neural network, that tries to find the optimal policy, evaluated only on a single trajectory $\tau$.
The detailed smoothness assumptions on $L_h$ that need to be imposed are discussed in Section~\ref{assumptions}. 
The gradient of  $\rL_h(\Bth, \Bom, \Bz)$ will be denoted by $h(\Bth, \Bom, \Bz).$

\item The second loss is given by $\rL_g(\Bth, \Bom, \Bz) = \EXP_{\tau \sim \pi(\Bth, \Bz)} 
    \left[ \Phi(g(\tau ; \Bom, \Bz) ,\tau,\Bth, \Bom,\Bz) \right]$ and is minimized with respect 
to $\Bom$ in order to find an optimal critic function 
$g( \tau ; \Bom, \Bz )$. The functions $g$ and $\Phi$ should again be sufficiently smooth, such that $\rL_g(\Bth, \Bom, \Bz)$ satisfies (L1)--(L3) from Section~\ref{assumptions}. $\Phi$ can be seen as the per-sample loss for the critic $g$. The gradient of $\rL_g$ will be denoted by $f$.

\item Since the expectations cannot be computed analytically, 
we do not have the exact gradients $h(\Bth, \Bom, \Bz)$ and $f(\Bth, \Bom, \Bz)$.
Therefore, the expectations are approximated by 
sampling sequences $\tau$ and computing the average gradient on
the sampled sequences.
In our case, the stochastic approximations $\hat{h}$ and $\hat{f}$
of the gradients $h$ and $f$ respectively, are created by randomly inserting only one sample trajectory $\tau$, i.e. we are dealing with online stochastic gradients. A formal description of the sampling process can be found in Section~\ref{martingale}. Our losses are then minimized using online stochastic gradient descent (SGD) with learning rates $a(n)$ and $b(n)$, where the integer $n\geq 0$ denotes the timestep of our iteration.
\end{itemize}
For a given $n$, let us now state the discussed building blocks in a more compact and formal way:

\begin{align}
\label{lossha}
    &\rL_h(\Bth_n, \Bom_n, \Bz_n)  \ =  \  \EXP_{\tau \sim \pi(\Bth_n,\Bz_n) } 
    \left[ \phi(\tau,\Bth_n, \Bom_n,\Bz_n) \right]  \ ,  \\ \nonumber
    &h(\Bth_n, \Bom_n, \Bz_n) \ =  \  \EXP_{\tau \sim \pi(\Bth_n,\Bz_n) } 
    \left[ \nabla_{\theta_n} \log \pi(\Bth_n,\Bz_n) \ \phi(\tau,\Bth_n, \Bom_n,\Bz_n)
    \ + \ \nabla_{\theta_n} \phi(\tau,\Bth_n, \Bom_n,\Bz_n) \right] \ ,
\end{align}

where the first possibility for the policy loss $L_h$ and its gradient $h$ is listed. $h$ is computed by the Policy Gradient Theorem, which can be found e.g. in \citet[Chapter~13.2]{Sutton:18book}. Next we discuss the expressions for the second possibility for $L_h$ (when sampling via a behavioral policy $\breve{\pi}$):

\begin{align}
\label{losshb}
    &\rL_h(\Bth_n, \Bom_n, \Bz_n)  \ =  \  \EXP_{\tau \sim  \breve{\pi}} 
    \left[ \phi(\pi(.;\Bth_n, \Bz_n),\tau,\Bth_n, \Bom_n,\Bz_n) \right]  \ ,  \\ \nonumber
     &h(\Bth_n, \Bom_n, \Bz_n) \ =  \  \EXP_{\tau \sim  \breve{\pi} } 
    \left[ \nabla_{\theta_n}  \phi(\pi(.;\Bth_n, \Bz_n),\tau,\Bth_n, \Bom_n,\Bz_n) \right] \ . \\
\label{lossg}
\end{align}

The expressions for our second loss $L_g$  and its gradient $f$ are as follows:

\begin{align}
    &\rL_g(\Bth_n, \Bom_n, \Bz_n) \ =  \ \EXP_{\tau \sim \pi(\Bth_n, \Bz_n)} 
    \left[ \Phi(g(\tau ; \Bom_n, \Bz_n) ,\tau,\Bth_n, \Bom_n,\Bz_n) \right]  \ , \\ \nonumber
     &f(\Bth_n, \Bom_n, \Bz_n) \ =  \  \EXP_{\tau \sim \pi(\Bth_n,\Bz_n) } 
    \left[\nabla_{\omega_n}  \Phi(g(\tau ; \Bom_n, \Bz_n) ,\tau,\Bth_n, \Bom_n,\Bz_n)  \right] \ , 
\end{align}

and finally, the iterative algorithm that optimizes the losses by online SGD, is given by:

\begin{align}
\label{iter1a}
 &\Bth_{n+1}  \  =  \ \Bth_n  \ -  \ a(n)  \ \hat{h} (\Bth_n, \Bom_n, \Bz_n) \ , \quad 
 \Bom_{n+1}  \ =  \ \Bom_n  \  -  \ b(n) \ \hat{f}(\Bth_n, \Bom_n, \Bz_n)  \ .
\end{align}

\paragraph{Main Theorem.}
Our main result will guarantee local convergence for \eqref{iter1a}. To this end we fix a starting point $(\Bth_0, \Bom_0)$ and determine an associated neighborhood $V_0 \times U_0$ which can be constructed by the loss assumptions (L1)--(L3) given in Section~\ref{assumptions}. The iterates \eqref{iter1a} will always stay in $V_0 \times U_0$ by these assumptions. Furthermore, let us denote the loss functions that result after considering the ``limit'' of the control sequence $\Bz_n \to \Bz$ by $L_h(\Bth,\Bom)$ (and similarly $L_g(\Bth,\Bom)$). Again we refer to Section \ref{borkarsetting} for a precise account on this informal description. Moreover, we denote a local minimum of $L_g(\Bth,\cdot)$ by $\Bla(\Bth)$, whereas $\Bth^*(\Bth_0, \Bom_0))$ should indicate a local minimum of $L_h(\cdot,\Bla(\cdot))$ in $V_0 \times U_0$. Also here we refer to Section ~\ref{assumptions} for a precise discussion. We can now state our main theorem:
\begin{theorem}
\label{th:convergence}
 Fix a starting point $(\Bth_0, \Bom_0)$. 
 Determine the associated neighborhood $V_0 \times U_0$ as in Section~\ref{assumptions}.
 Assume learning rates like (A4) for the time-scales $a(n)$ and $b(n)$ mentioned in 
 Section~\ref{borkarsetting}. Also take the loss assumptions in Section ~\ref{assumptions} for granted. \\ \noindent
Then Eq.~\eqref{iter1a}
converges to a local minimum $(\Bth^*(\Bth_0, \Bom_0),\Bla(\Bth^*(\Bth_0, \Bom_0)))$ of the associated losses  (i) Eq.~\eqref{lossha} or (ii) Eq.~\eqref{losshb} and Eq.~\eqref{lossg}: 
$
(\Bth_n, \Bom_n) \ \to \ (\Bth^*(\Bth_0, \Bom_0),\Bla(\Bth^*(\Bth_0, \Bom_0)))
\ \ \text{a.s.}
$
as $ n \ \to \ \infty. $
\end{theorem}

\subsection{Convergence Proof for PPO.}
The main theorem is applied to prove convergence of PPO.
Our PPO variant uses deep neural networks, 
softmax outputs for the policy network, 
regularization, trust region, or exploration terms.
Regularization can be entropy, weight decay, 
or a trust region penalty like the Kullback-Leibler divergence. 
All functions are assumed to be sufficiently smooth, i.e.\ 
at least three times continuously differentiable 
and bounded wrt. the parameters. 
The losses should satisfy (L1)--(L3) from Section~\ref{assumptions}. 
The PPO alorithm aims at minimizing the following losses:

\begin{align} 
 \label{lossh2}
    &\rL_h(\Bth_n, \Bom_n, \Bz_n)  \ =  \  \EXP_{\tau \sim \pi(\Bth_n,\Bz_n) } 
    \left[ - \ G_0 \ + \ (z_2)_n \ \rho(\tau,\Bth_n,\Bz_n) \right]  \ , \\
    % \nonumber
 %   &h(\Bth_n, \Bom_n, \Bz_n) \ =  
%    \EXP_{\tau \sim \pi(\Bth_n,\Bz_n) } 
%    \left[ - \ \sum_{t=0}^T \nabla_{\Bth} \ln \pi(a_t \mid s_t ; \Bth_n,\Bz_n) \ 
%    \left( \hat{q}^{\pi}(s_{t},a_{t}) \ - \ \hat{v}^\pi(s_t) \right) \right. \\ \nonumber
%     &\left. + \ (z_2)_n \  \sum_{t=0}^T \nabla_{\Bth_n} \ln \pi(a_t \mid s_t ; \Bth_n,\Bz_n) 
%     \ \rho(\tau,\Bth_n,\Bz_n)
%     \ + \ (z_2)_n  \ \nabla_{\Bth_n}  \rho(\tau,\Bth_n,\Bz_n) \right] \ , \\
 \label{lossg2}
     &\rL^\mathrm{TD}_g(\Bth_n, \Bom_n, \Bz_n) \ = 
     \EXP_{\tau \sim \pi(\Bth_n,\Bz_n)} \left[\frac{1}{2} \
      \sum_{t=0}^{T} \big( \delta^{\mathrm{TD}}(t) \big)^2 \right]  \ , \\
  %  &f^\mathrm{TD}(\Bth_n, \Bom_n, \Bz_n) \ =       
%    \EXP_{\tau \sim \pi(\Bth_n,\Bz_n)} \left[
 %     \sum_{t=0}^{T}   \delta^{\mathrm{TD}}(t) \ \nabla_{\Bom_n} \hat{q}^{\pi}(s_t,a_t;\Bom_n) \right]  \ , \\
 \label{lossg2MC}
     &\rL^\mathrm{MC}_g(\Bth_n, \Bom_n, \Bz_n) \ = 
     \EXP_{\tau \sim \pi(\Bth_n,\Bz_n)} \left[\frac{1}{2} \
      \sum_{t=0}^{T} \bigg( G_t \ - \ \hat{q}^{\pi}(s_t,a_t;\Bom_n) \bigg)^2 \right]  \ , \\
      %\nonumber 
 %   &f^\mathrm{MC}(\Bth_n, \Bom_n, \Bz_n) \ = \  \EXP_{\tau \sim \pi(\Bth_n,\Bz_n)} \left[
 %     \sum_{t=0}^{T} \bigg(  G_t \ - \ \hat{q}^{\pi}(s_t,a_t;\Bom_n) \bigg) \ \nabla_{\Bom_n} \hat{q}^{\pi}(s_t,a_t;\Bom_n) \right]  \ , \\
\label{iter1PPO}
 &\Bth_{n+1}  \  =  \ \Bth_n  \ -  \ a(n)  \ \hat{h} (\Bth_n, \Bom_n, \Bz_n) \ , \quad
 \Bom_{n+1}  \ =  \ \Bom_n  \  - \ b(n) \ \hat{f}(\Bth_n, \Bom_n, \Bz_n) \ .
\end{align}

Let us now briefly describe the terms in Eq. \eqref{lossh2}--\eqref{iter1PPO}:
\begin{itemize}[topsep=-5pt,itemsep=1pt,partopsep=0pt, parsep=0pt,wide, labelwidth=!, labelindent=0pt]
\item $\hat{q}^\pi(s_t, a_t; \Bom)$ is a function that approximates the $Q$-value $q^\pi(s_t,a_t)$.%= \EXP_\pi[ \sum_{\tau=t}^T R(s_{\tau},a_{\tau}) \mid s_t,a_t]$.
\item $\delta^{\mathrm{TD}}(t) =  R(s_t,a_t)  + \hat{q}^{\pi}(s_{t+1},a_{t+1};\Bom_{n-1} )  -    \hat{q}^{\pi}(s_t,a_t;\Bom_n )$ is the temporal difference error.
\item The exact gradients $h(\Bth_n, \Bom_n, \Bz_n)$ 
    (from the Policy Gradient Theorem assuming causality and subtracting a baseline \citep[Chapter~13.2]{Sutton:18book}),  
    $f^\mathrm{TD}(\Bth_n, \Bom_n, \Bz_n)$ and  
    $f^\mathrm{MC}(\Bth_n, \Bom_n, \Bz_n)$ of the respective losses
    can be found in Sections \ref{sec:ppo} and \ref{causality} in the appendix.
    \item $\Bz_n=((z_1)_n,(z_2)_n,(z_1)_{n-1},(z_2)_{n-1},\Bth_{n-1},\Bom_{n-1})$ 
     denotes an additional controlled Markov process with values in compact sets.
     The controlled Markov process is essential 
     to define the trust region term of PPO which uses previous values of $\Bth$ and $z_1$. Here, 
     $z_1 \in [1,\beta]$ increases from $1$ to $\beta>1$ and $z_2 \in [0,(z_2)_0]$ 
     decreases from $(z_2)_0>1$ to $0$. $z_1$ controls the amount of greediness and $z_2$ the regularization. Details can be found in Section \ref{sec:control_details}. %; concerning the choice of $\beta$ we refer to Sections~\ref{sec:finte_greedy} and~\ref{sec:fgreedy} in the appendix.
     \item $\pi(\Bth_n,\Bz_n)$ is a softmax policy that depends on $(z_1)_n$ to make it more greedy. We will introduce it precisely in Section \ref{sec:control_details}, especially  Eq.~\eqref{eq:policy} there.
      $\pi$ is learned using $\hat{q}$ and updated in every time-step.
      \item $\rho(\tau,\Bth_n,\Bz_n)$ includes the trust region term of PPO and 
      may also include regularization terms like weight decay or entropy regularization.
      For example, 
      $\rho(\tau,\Bth_n,\Bz_n)= \rK\rL_{\epsilon}(\pi(\Bth_{n-1},(z_1)_{n-1}), \pi(\Bth_n,(z_1)_n))$,
      where $\rK\rL_{\epsilon}(\Bp,\Bq)=\rK\rL(\tilde{\Bp},\tilde{\Bq})$ with
       $\tilde{p}_i=(p_i+\epsilon)/(1+k \epsilon)$.  %    $\rK\rL_{\epsilon}$ is the $\epsilon$ regularized Kullback-Leibler divergence. %$\rho$ is assumed to be three times continuously differentiable and bounded. 
\end{itemize}
The next corollary states that the above described PPO algorithms (TD and MC versions) converge.
\begin{corollary}[Convergence PPO]
 Fix a starting point $(\Bth_0, \Bom_0)$. 
 Determine the associated neighborhood $V_0 \times U_0$ as in Section~\ref{assumptions}.
 Assume learning rates like (A4) for the time-scales $a(n)$ and $b(n)$ mentioned in 
 Section~\ref{borkarsetting}. Also take the loss assumptions in Section ~\ref{assumptions} for granted. \\ \noindent
 Using the same notation as in Theorem~\ref{th:convergence}, the PPO algorithm Eq.~\eqref{iter1PPO} 
 converges to a local minimum $(\Bth^*(\Bth_0, \Bom_0),\Bla(\Bth^*(\Bth_0, \Bom_0)))$ of the associated losses Eq.~\eqref{lossh2} and either Eq.~\eqref{lossg2} or Eq.~\eqref{lossg2MC}: 
$
(\Bth_n, \Bom_n) \ \to \ (\Bth^*(\Bth_0, \Bom_0),\Bla(\Bth^*(\Bth_0, \Bom_0)))
\ \ \text{a.s.}
$
as $ n \ \to \ \infty $.
\end{corollary}
\begin{proof}
 We apply Theorem~\ref{th:convergence} since all its assumption are fulfilled.
\end{proof}

\subsection{Convergence Proof for RUDDER.}
The main theorem is applied to prove convergence of RUDDER,
which excels for tasks with sparse and delayed rewards \citep{Arjona-Medina:18}. For a recent application, see \citet{Patil:20}.
Again, we assume enough smoothness for all functions, 
i.e.\ they are at least three times continuously differentiable wrt. 
the parameters and bounded. 
The losses should satisfy (L1)--(L3) from Section~\ref{assumptions}.
We formulate the RUDDER algorithm as a minimization problem of square losses $\rL_h(\Bth_n, \Bom_n,\Bz_n)$ and $\rL_g(\Bth_n, \Bom_n, \Bz_n)$:
\begin{align} \label{lossh1}
    &\rL_h \ = \
    \EXP_{\tau \sim \breve{\pi}}
    \left[ \frac{1}{2} \ \sum_{t=0}^{T} \bigg( R_{t+1}(\tau ; \Bom_n)  -   \hat{q}(s_t, a_t; \Bth_n)\bigg)^2 \ + \ (z_2)_n \ \rho_{\Bth}(\tau,\Bth_n,\Bz_n) \right] \ , \\ 
      \label{lossg1}
     &\rL_g \ = \
     \EXP_{\tau \sim \pi(\Bth_n,\Bz_n)} \left[\frac{1}{2} \
     \bigg( \sum_{t=0}^{T} \tilde{R}_{t+1} \ - \   g( \tau ;
     \Bom_n ) \bigg)^2  \ + \ (z_2)_n \ \rho_{\Bom}(\tau,\Bth_n,\Bz_n) \right] \ . \\ 
\label{iter1RUDDER}
 &\Bth_{n+1}  \ =  \ \Bth_n  \ -  \ a(n)  \ \hat{h} (\Bth_n, \Bom_n, \Bz_n) \ , \quad
 \Bom_{n+1}  \ =  \ \Bom_n  \  -  \ b(n) \ \hat{f}(\Bth_n, \Bom_n, \Bz_n)  \ .
\end{align}
Let us now briefly describe the terms in Eq. \eqref{lossh1}--\eqref{iter1RUDDER}:
%MH: Maybe we can add some comments regarding the regularization?
\begin{itemize}[topsep=-5pt,itemsep=1pt,partopsep=0pt, parsep=0pt,wide, labelwidth=!, labelindent=0pt]
\item $\hat{q}(s,a;\Bth_n)$ is a function parametrized by $\Bth_n$ 
          that approximates the $Q$-value $q(s,a)$.
          Note that the policy loss $L_h$ implicitly depends on the policy $\pi$ via $\hat{q}$.
\item The expressions for $h(\Bth_n, \Bom_n, \Bz_n)$ and $f(\Bth_n, \Bom_n, \Bz_n)$ can be found in Section \ref{sec:rudder} in the appendix.
    \item $\tilde{R}$ is the original MDP reward.
    \item $R(\tau;\Bom_n)$ is the redistributed reward based 
          on the return decomposition of $g$ with parameter vector 
          $\Bom_n$. For a state-action sequence $\tau$ the realization of its redistributed reward $R$ is computed from
          $g(\tau; \Bom_n )$ %, $g(\tau; \Bom ), t \geq 0 $, 
          and the realization of return variable $\sum_{t=0}^T \tilde{R}_{t+1}$. In practice, $g$ can be an LSTM-network, or, e.g. in \citet{Patil:20}, $g$ is obtained by a profile model.
   % \item $g$ denotes the return decomposition function parametrized by $\Bom_n$.
    %      It predicts the return $\sum_{t=0}^T \tilde{R}_{t+1}$ given a complete state-action sequence
    %      $\tau$. In practise $g$ is an LSTM-network.
    %\item $\hat{q}(s,a;\Bth_n)$ is a function parametrized by $\Bth_n$ 
    %      that approximates the $Q$-value $q(s,a)$.%=\EXP\left[R \mid s,a \right]$. It is
    \item $\rho_{\Bth}(\tau,\Bth_n,\Bz_n)$ is a regularization term for 
    learning the $Q$-value approximation $\hat{q}$. 
    \item $\rho_{\Bom}(\tau,\Bth_n,\Bz_n)$ is a regularization term for 
    learning the reward redistribution function $g$. 
    \item $\breve{\pi}$ is a behavioral policy that does not depend on the parameters.
   %  \item $\EXP_{\pi(\Bth_n,\Bz_n)}\left[ . \right]$ ($\EXP_{\breve{\pi}}\left[ . \right]$) is an expectation over sequences $\tau=(s_0,a_0,s_1,a_1,\ldots,s_T,a_T)$, which are sampled with policy $\pi(\Bth_n,\Bz_n)$ ($\breve{\pi}$).
     \item $\Bz_n=((z_1)_n,(z_2)_n,(z_1)_{n-1},(z_2)_{n-1},\Bth_{n-1},\Bom_{n-1})$ 
     denotes an additional Markov process, where we use the same construction as in the PPO setting.
     Details can again be found in Section \ref{sec:control_details}.
     %$z_1 \in [1,\beta]$ increases from $1$ to $\beta>1$, $z_2 \in [0,(z_2)_0]$ 
     %decreases from $(z_2)_0>1$ to $0$.
     %and the parameters $\Bth$ and $\Bom$ are assumed to be from a compact set.
     \item $ \pi(\Bth_n,\Bz_n) $ is a softmax policy applied to
     $(z_1)_n \ \hat{q}$ (see Eq.~\eqref{eq:policy} for a precise introduction).
     It depends on  $(z_1)_n>1$ which makes it more greedy and $\Bth_n$ is updated in every time-step.  
\end{itemize}
The next corollary states that the RUDDER algorithm converges.
\begin{corollary}[Convergence RUDDER]
 Fix a starting point $(\Bth_0, \Bom_0)$. 
 Determine the associated neighborhood $V_0 \times U_0$ as in Section~\ref{assumptions}.
 Assume learning rates like (A4) for the time-scales $a(n)$ and $b(n)$ mentioned in Section~\ref{borkarsetting}.  
 Also take the loss assumptions in Section ~\ref{assumptions} for granted.
 Using the same notation as in Theorem~\ref{th:convergence}, the RUDDER algorithm Eq.~\eqref{iter1RUDDER} 
 converges to a local minimum $(\Bth^*(\Bth_0, \Bom_0),\Bla(\Bth^*(\Bth_0, \Bom_0)))$ of the associated losses Eq.~\eqref{lossh1} and Eq.~\eqref{lossg1}: 
$
(\Bth_n, \Bom_n) \ \to \ (\Bth^*(\Bth_0, \Bom_0),\Bla(\Bth^*(\Bth_0, \Bom_0)))
\ \ \text{a.s.}
$
as $ n \ \to \ \infty $.
\end{corollary}
\begin{proof} 
We apply Theorem~\ref{th:convergence} since all its assumptions are fulfilled.
\end{proof}
%MH: should we introduce the setting with just one control z_n?

\section{Assumptions and Proof of Theorem~1}
This section aims at presenting the theoretical framework from \citet{Karmakar:17} and \citet{Borkar:08} that we want to apply to prove Theorem~\ref{th:convergence}.
We formulate the convergence result Theorem~\ref{th:karmakar} and the assumptions (A1)--(A7) that we need to ensure in order to get there. Then we discuss how it can be applied to our setting.
%The next paragraphs give the assumptions that we have to make for Theorem~\ref{th:convergence}
%in order to fulfill assumptions (A1)--(A7) and to apply Theorem~\ref{th:karmakar}.
%MH: i commented this out, reader will find out anyway. can of course also leave it if you want.
%These paragraphs are:
%"Assumption (A1): Definition of the Markov Control",
%"Assumption (A3): martingale Difference Property and the Probabilistic Setting",
%"Assumption (A2) and Assumption (A6): Ensuring Smooth and Lipschitz $h$ and $f$ and Convergence of the ODEs
%via Assumptions on the Losses", and  
%"Assumption (A5) and Assumption (A7): Transition Kernel and Bounded Iterates."
%SEPP: above is much text with few content: can be much reduced.
%In the corresponding paragraphs we ensure that all assumptions (A1)--(A7) of Theorem~\ref{th:karmakar} are
%fulfilled either by showing properties or making assumptions on our setting.
%Then we prove Theorem~\ref{th:convergence} using the assumptions of the paragraphs.
%In the final Section we show that 
%finite greediness is sufficient to allow to converge to the optimal policy.
%Our convergence proofs show convergence but not necessarily to the optimal policy,
%but convergence to the optimal policy should be possible.

\subsection{The Stochastic Approximation Theory: Borkar and Karmakar \& Bhatnagar.} 
\label{borkarsetting}

%Here we summarize the required background from stochastic approximation theory \citep{Karmakar:17,Borkar:08}
%and ODE theory \citep{Teschl:12book}. 
%The aim is to give an overview that suffices to understand the main proof ideas for Theorem~\ref{th:convergence}. Further technical details can be found in the supplements or the original papers.
For this section we use the formulations in \citet{Heusel:17,Karmakar:17}.
Stochastic approximation algorithms are iterative procedures 
to find stationary points (minimum, maximum, saddle point) of
functions when only noisy observations are provided. 
We use two time-scale stochastic approximation algorithms, 
i.e.\ two coupled iterations moving at different speeds. 
Convergence of these interwoven iterates can be ensured by assuming that one
step size is considerably smaller than the other. 
The slower iterate is assumed to be
slow enough to allow the fast iterate to converge while simultaneously being perturbed
by the slower. The perturbations of the slower should be small
enough to ensure convergence of the faster.
The iterates map at time step $n\geq 0$ the fast variable $\Bom_n \in \dR^k$ and the slow
variable $\Bth_n \in \dR^m$ to their new values:
%\begin{align}
%\label{iter1}
% \Bth_{n+1}  \ &=  \  \Bth_n  \ - \  a(n) \ (h(\Bth_n, \Bom_n,
%               (\Bz_1)_n ) \  +  \ (\Bm_1)_n ) \ ,\\
%\label{iter2}
% \Bom_{n+1}  \ &=  \ \Bom_n   \ - \  b(n) \ (f(\Bth_n, \Bom_n,
%              (\Bz_2)_n )  \ + \ (\Bm_2)_n) \ ,
%\end{align}
\begin{align}
\label{iter1}
 \Bth_{n+1}  \ &=  \  \Bth_n  \ - \  a(n) \ (h(\Bth_n, \Bom_n,
               \Bz_n ) \  +  \ (\Bm_1)_n ) \ ,\\
\label{iter2}
 \Bom_{n+1}  \ &=  \ \Bom_n   \ - \  b(n) \ (f(\Bth_n, \Bom_n,
               \Bz_n )  \ + \ (\Bm_2)_n) \ ,
\end{align}
where:
\begin{itemize}[topsep=-5pt,itemsep=1pt,partopsep=0pt, parsep=0pt,wide, labelwidth=!, labelindent=0pt]
\item $h(.)\in \dR^m$ and $f(.)\in \dR^k$ are 
mappings for  Eq.~\eqref{iter1} and Eq.~\eqref{iter2}, respectively.
\item $a(n)$ and $b(n)$ are
step sizes for Eq.~\eqref{iter1} and Eq.~\eqref{iter2}, respectively. 
\item $(\Bm_1)_n$ and $(\Bm_2)_n$ are martingale difference sequences for Eq.~\eqref{iter1} and Eq.~\eqref{iter2}, respectively.
%\item $(\Bz_1)_n$ and $(\Bz_2)_n$ denote Markov control processes for Eq.~\eqref{iter1} and Eq.~\eqref{iter2}, respectively.
\item $\Bz_n$ denotes the common  Markov control process for Eq.~\eqref{iter1} and Eq.~\eqref{iter2}.
\end{itemize}
We assume that all the random variables are defined 
on a common probability space $(\Omega,\mathfrak{A},P)$ with associated sigma algebra $\mathfrak{A}$ and probability measure $P$. Let us continue with an informal summary of the assumptions needed to ensure convergence of \eqref{iter1}--\eqref{iter2}. More precise technical details can be found in Section \ref{ch:precise} in the appendix and in \citet{Karmakar:17}.
\begin{enumerate}[wide, labelwidth=!, labelindent=0pt,label=\textbf{\large (A\arabic*)}]
\item \textit{Assumptions on the controlled Markov processes:} 
%For $i=1,2$, the controlled Markov processes
%$(\Bz_i)_n\}$ take values in a compact metric space $S^{(i)}$.
%Both processes are controlled by the iterate sequences $\{\Bth_n\}$
%and $\{\Bom_n\}$ and $\{(\Bz_i)_n\}$ are additionally
%controlled by a random process $\{(\Ba_i)_n\}$ taking values in a
%compact metric space $U^{(i)}$.
%The dynamics wrt.\ $n$ is specified by a transition kernel. Be aware that the setting discussed here can in general be different than the one introduced in the MDP setting.
$\Bz_n$ takes values in a compact metric space $S$.
It is controlled by the iterate sequences $\Bth_n$
and $\Bom_n$ and additionally
by a random process $\Ba_n$ taking values in a
compact metric space $W$.
The dynamics wrt.\ $n$ are specified by a transition kernel. Be aware that this control setting can in general be different than the already introduced MDP setting.
\item \textit{Assumptions on the update functions:}
$f$, and $h$ are jointly continuous as well as Lipschitz in 
their first two arguments, and uniformly w.r.t.\ the third.
\item \textit{Assumptions on the additive noise:}
For $i=1,2$ the $(\Bm_i)_n$ are martingale difference sequences with bounded second moments.
\item 
\textit{Assumptions on the learning rates:}
Informally, the sums of the positive $a(n)$ and $b(n)$ diverge, while
  their squared sums converge. $a(n)$ goes to zero faster than $b(n)$. 
%These assumptions are standard in the literature.
\item \textit{Assumptions on the transition kernels:}
The transition kernels of $\Bz_n$ are continuous wrt. the topology of weak convergence of probability measures.
\item \textit{Assumptions on the associated ODEs:}
We consider occupation measures
which intuitively give for 
the controlled Markov process the probability or 
density to observe a particular
state-action pair from $S \times W$ 
for given $\Bth$ and $\Bom$ and a given control. %policy $\pi$. 
A precise definition of these occupation measures can be found e.g.\ 
on page~68 of \citet{Borkar:08} or page~5 in \citet{Karmakar:17}.
We need the following assumptions:
\begin{itemize}[topsep=-5pt,itemsep=1pt,partopsep=0pt, parsep=0pt,wide, labelwidth=!, labelindent=0pt]
\item We assume that there exists only one such ergodic occupation measure for $\Bz_n$ on $S \times W$, denoted by $\Gamma_{\Bth,\Bom}$.
%for the prescribed $\Bth$ and $\Bom$ 
A main reason for assuming uniqueness is that it enables us to deal with ODEs instead of differential inclusions.
Moreover, set 
$\tilde{f}(\Bth, \Bom) \ = \ \int f(\Bth,\Bom,\Bz) \ \Gamma_{\Bth,\Bom}(\Rd\Bz, W)$.
\item For $ \Bth \in \dR^m$, the ODE
$
\dot{\Bom}(t) \ = \ \tilde{f}(\Bth,\Bom(t)) 
$
has a unique asymptotically stable equilibrium $\Bla(\Bth)$ with attractor set $B_{\Bth}$ 
such that $\Bla :  \dR^m \to \dR^k$ is a Lipschitz map with global Lipschitz constant.
\item The Lyapunov function $V(\Bth,.)$ associated to $\Bla(\Bth)$ is continuously differentiable.
\item Next define
$
\tilde{h}(\Bth) \ = \ \int h(\Bth,\Bla(\Bth),\Bz) \ \Gamma_{\Bth,\Bla(\Bth)}(\Rd\Bz, W).
$
The ODE
$
\dot{\Bth}(t) \ = \ \tilde{h}(\Bth(t)) 
$
has a global attractor set $A$.
\item For all $\Bth$, with probability 1, $\Bom_n$ for $n\geq 1$ 
belongs to a compact subset $Q_{\Bth}$ of $B_{\Bth}$ ``eventually''.
\end{itemize}
\vspace{2mm}
This assumption is an adapted version of (A6)' of \citet{Karmakar:17} 
to avoid too many technicalities (e.g. \citet{Karmakar:17} uses a different control for each iterate).
\item \textit{Assumption of bounded iterates:} The iterates $\Bth_n$ and $\Bom_n$ are uniformly bounded almost surely.

\end{enumerate}
Convergence for Eq.~\eqref{iter1}--\eqref{iter2} is given by Theorem~1 in \citet{Karmakar:17}:
\begin{theorem}[Karmakar \& Bhatnagar]
\label{th:karmakar}
Under the assumptions (A1)--(A7), 
%that are given informally below and precisely in \citep{Karmakar:17}, 
the iterates
Eq.~\eqref{iter1} and Eq.~\eqref{iter2} converge:
$
(\Bth_n, \Bom_n) \ \to \ \cup_{\Bth^* \in A}(\Bth^*, \Bla(\Bth^*))
\ \ \text{a.s.} \quad \text{as} \ n \ \to \ \infty \ .
$ 
%MH: can also remove the next line, A is now properly introduced
%(Recall: $A$ is a global attractor associated to $\tilde{h}$).
\end{theorem}
\subsection{Application to Proof of Main Result}
Next we describe how Theorem~\ref{th:karmakar} yields Theorem~\ref{th:convergence} by discussing the validity of (A1)--(A7). We additionally mention details about their concrete realization in the context of PPO and RUDDER. We conclude by a discussion on how we can allow our policies to become sufficiently greedy over time.

%\fbox{MH: would also leave this}
%% This could be an option to remove

\paragraph{(A1) Controlled Markov process for the abstract setting:}
\label{control}
For Eq.~\eqref{lossha} -- \eqref{lossg} we assume to have a controlled process that fulfills the previously discussed requirements for (A1). 
%However, also to not overload notation, we do not distinguish between $(\Bz_i)_n$, $i=1,2$ and will use the same process for both iterates, which we denote by $\Bz_n$.
%This could in principle also be generalized in an obvious way.
\paragraph{(A1)  Controlled Markov process for PPO and RUDDER:}
\label{sec:control_details}
In our applications to RUDDER and PPO, however, the Markov control will have a much simpler form: $\Bz_n$ mainly consists of real sequences which obey the Markov property. Also we do not have any additional control in these situations. 
%If not stated otherwise, by abuse of notation we will 
%in the sequel not distinguish between $\Bz_n$ and $\B_n$. 
More concretely: 
$\Bz_n=((z_1)_n,(z_2)_n,(z_1)_{n-1},(z_2)_{n-1},\Bth_{n-1},\Bom_{n-1})$ with 
$(z_1)_n \in [1,\beta]$ for some $\beta>1$, and $(z_2)_n \in [0,(z_2)_0]$ for some $(z_2)_0>1$.
%and the parameters $\Bth$ and $\Bom$ are assumed to be in a compact set.
$(z_1)_n$ can be defined by $(z_1)_0=1$ 
and $(z_1)_{n+1}=(1-\frac{1}{\beta})(z_1)_n+1$.
It consists of the partial sums of a geometric series converging to $\beta>1$.
For $(z_2)_n$ we can use any sequence satisfying the Markov Property and converging to zero, e.g. $(z_2)_0=C$ 
and $(z_2)_{n+1}= \alpha (z_2)_n$ with $\alpha<1$ or
$(z_2)_{n+1}= \frac{(z_2)_n}{(z_2)_n + \alpha}$ with $1<\alpha$.
$\Bz_n$ then is a time-homogeneous Markov process with unique invariant measure, cf. Section \ref{sec:dettheo1} in the appendix.

Let us now describe the meaning of this process for RUDDER and PPO: 
The component $(z_1)_n$ is used as a slope parameter for the softmax policy and 
goes to a large value $\beta$ to make the policy greedy. The softmax policy is introduced in the following rather abstract way:
For a sufficiently smooth function (deep neural network) $\Bps(s;\Bth)=(\psi^1(s;\Bth), \ldots, \psi^{|\sA|}(s;\Bth))$, a softmax policy $\pi(\Bth,z_1)$ is defined as
\begin{align}
 \label{eq:policy}
  \pi(a^i \mid s;\Bth,z_1) 
  \ &= \ \frac{ \exp(z_1 \ \psi^i(s;\Bth)) }
   {\sum_j \exp(z_1 \ \psi^j(s;\Bth) )} \ .
\end{align}
%Since $z_1$ is bounded, $\pi$ still fulfills our assumptions 
%on the policy function and the associated losses.
%(see Supplement).
For RUDDER and PPO we use $\Bps(s;\Bth)=\hat{q}(s;\Bth_n)$ with $\hat{q}^i(s;\Bth_n)$ approximating
$q^{\pi}(s,a^i)$. 
The component $(z_2)_n$ is used to
weight an additional term in the objective and goes to zero over time.
We require $(z_1)_{n-1}$ and $\Bth_{n-1}$ for the trust-region term $\rho$, 
for which we have to reconstruct the old policy.
Further details, especially concerning $\beta$, can be found in Section \ref{sec:finte_greedy} and Section \ref{sec:fgreedy} in the appendix.

\paragraph{(A3) Martingale Difference Property and the Probabilistic Setting.} 
\label{martingale}

%Here we introduce the required baseline probability space $\Omega$
%and the martingale difference sequences $(\Bm_i)_{n+1}$, $i=1,2$.
Here we describe the sampling process more formally: 
%that induces the randomness, in several steps:
\begin{itemize}[topsep=-5pt,itemsep=1pt,partopsep=0pt, parsep=0pt,wide, labelwidth=!, labelindent=0pt]
\item The baseline probability space is given by 
$\Omega=[0,1]$, $P=\mu$ and $\mathfrak{A}=\mathfrak{B}([0,1])$, with $\mu$ denoting 
the Lebesgue measure and 
$\mathfrak{B}([0,1])$ the Borel $\sigma$-algebra on $[0,1]$.
\item Next we introduce the set of all trajectories obtained by following $\pi$ as
$
\tilde{\Omega}_{\pi} = \left\{ \tau=(s,a)_{0:T} | \tau \text{ is chosen wrt. } \pi, S_0 = s_0, A_0 = a_0 \right\}.
$
Its power set serves as related $\sigma$-algebra $\tilde{\mathfrak{A}}_{\pi}$.
%\fbox{LG: %$\tilde{\mathfrak{A}_{\pi}}$ instead?}.
%MH: Changed it, thanks!
\item We endow $\tilde{\mathfrak{A}}_{\pi}$ with 
a probability measure:  %$\tilde{P_{\pi}}$:
$
 \tilde{P_{\pi}}(\tau) \ = \ \prod_{t=1}^{T} p(s_t \mid s_{t-1},a_{t-1}) \  \pi(a_t \mid s_t),
$
which computes the probability of choosing a sequence $\tau$ 
with starting point $(s_0,a_0)$.
$\tilde{\mathfrak{A}_{\pi}}$ can be ordered according to the magnitude of the values of its events on $P_{\pi}$. We denote this ordering by $\le$.
\item We define $S_{\pi}: \Omega \to \tilde{\Omega}_{\pi}$ as
$ S_{\pi}:
x \mapsto \argmax_{\tau \in \tilde{\Omega}_{\pi}} \left\{ \sum_{\eta\le \tau} \tilde{P_{\pi}}(\eta) \le x \right\}.
$
This map is well defined and measurable and 
it describes how to get one sample from 
a multinomial distribution with probabilities $\tilde{P_{\pi}}(\tau)$, 
where $\tau \in \tilde{\Omega}_{\pi}$.
%(see Section~S3 in the Supplements for more details).
%SEPP control whether this still holds: refer to supplement
%Here we only apply the inverse sampling method to our setting. 
%LG: Changed it from S2 to S3.
\item
Now we are in the position to describe the sampling process. As mentioned already in the beginning, we use an online update, i.e.\ we introduce functions $\hat{h}$ and $\hat{f}$, where $\hat{h}$
approximates $h$ by using one sample trajectory instead of the expectation, the same goes for $\hat{f}$. 
%For each time-step $n$, the sample for $\hat{h}$ is chosen according to 
%the policy $\pi(\Bth_n, \Bz_n)$  for Eq.~\eqref{lossha} and
%according to the behavioral policy $\breve{\pi}$ for Eq.~\eqref{losshb}, 
%whereas for $\hat{f}$ this is done according to $\pi(\Bth_n, \Bz_n)$.
More formally, 
for $f$ we define $\hat{f}(\Bth_n,\Bom_n,\Bz_n): [0,1] \to \mathbb{R}^k$ as
$
x \mapsto S_{\pi(\Bth_n, \Bz_n)}(x)=\tau 
\mapsto  \nabla_{\Bom_n}  \Phi(g(\tau ; \Bom_n, \Bz_n) ,\tau,\Bth_n, \Bom_n,\Bz_n).
$
\item Finally we can define the martingale errors as ${(\Bm_1)_{n+1} = \hat{h}(\Bth_n, \Bom_n,\Bz_n) -  h(\Bth_n, \Bom_n,\Bz_n)}$ and ${(\Bm_2)_{n+1} = \hat{f}(\Bth_n, \Bom_n, \Bz_n) - f(\Bth_n, \Bom_n, \Bz_n)}$.
%since $\sigma(\Bom_n)$ and $\sigma(\Bth_n)$ already describe $\tilde{\mathcal{A}}_{\pi_{\Bth_n}}$, and thus:
%\[
%\EXP[M^{(1)}_{n+1}|\mathcal{F}_n]=\EXP[\hat{f}(\Bth_n, \Bom_n, \Bz_n)|\mathcal{F}_n]-\tilde{f(\Bth_n, %\Bom_n, \Bz_n)}=0.
%\]
\end{itemize}
\vspace{1mm}
Further details (regarding the sampling process and  the bounds for the second moments) can be found in Sections \ref{sec:sampling} and \ref{sec:dettheo1} in the appendix.
\paragraph{(A2) and (A6) Smoothness for $h$ and $f$ and Stability of ODEs
via Assumptions on Losses.}
\label{assumptions}
%\fbox{MH: would also not move this to appendix}
%To apply Theorem~\ref{th:karmakar}, 
We make the following assumptions on 
the loss functions of Eq.~\eqref{lossha} (or Eq.~\eqref{losshb}) and Eq.~\eqref{lossg}:
\begin{enumerate}[wide, labelwidth=!, labelindent=0pt,label=\textbf{(L\arabic*)}]
\item  $\pi$, $g$, $\Phi$ and $\phi$ all have compact support and 
are at least three times continuously differentiable wrt.\ their parameters $\Bth$ and $\Bom$.
\item For each fixed $\Bth$ all critical points of $L_g(\Bth,\Bom)$ are isolated local minima and there are only finitely many. The local minima $\{ \lambda_i(\Bth) \}_{i=1}^{k(\Bth)}$ of $L_g(\Bth,.)$ can be expressed locally as at least twice continuously differentiable 
functions with associated domains of definitions $\{ V_{\lambda_i(\Bth)} \}_{i=1}^{k(\Bth)}$.
\item Locally in $V_{\lambda_i(\Bth)}$, 
$L_h(\Bth,\lambda_i(\Bth))$ has only one local minimum. 
\end{enumerate}

%Next we give some remarks concerning the practical relevance 
%of our assumptions as well as their role in the proof of our result:
Some remarks concerning these assumptions are in order:
\begin{itemize}[topsep=-5pt,itemsep=1pt,partopsep=0pt, parsep=0pt,wide, labelwidth=!, labelindent=0pt]
    \item \textit{Comment on (L1)}: 
    The parameter space of networks can be assumed to be bounded in practice.
  %  therefore, also compact.
    \item \textit{Comment on (L2)}: 
    For each starting point $(\Bth_0, \Bom_0)$ we can find a neighborhood $U_{\Bth_0}(\Bom_0)$ 
    that connects $\Bom_0$ with 
    a local minimum $\lambda_i(\Bth_0)$ of $L_g(\Bth_0,\cdot)$, so that $U_{\Bth_0}(\Bom_0)$ contains no further critical points, e.g. a small neighborhood around the steepest descent path on the loss surface of $L_g(\Bth_0,\cdot)$ starting at $\Bom_0$.
    %except if it traverses saddles, 
    %where $\lambda_i(\Bth_0)$ is the only local minimum.
   Next we apply the implicit function theorem (IFT)
    to find a neighborhood $V_0$ around $\Bth_0$, 
    such that $\lambda_i(\Bth)$ is twice continuously differentiable there.
    The IFT can be applied to $f(\Bth,\cdot)=\nabla_{\Bom}L_g(\Bth,\cdot)=0$, 
    since the associated Hessian is positive definite and thus invertible. 
    It can even be shown that it is twice continuously differentiable, 
    using analytic versions of the IFT.
    \item \textit{Comment on (L3)}: 
    In a similar vein, for each $\Bth \in V_0$ we can construct neighborhoods $U_{\Bth}(\Bom_0)$ around $\Bom_0$ with $\lambda_i(\Bth)$ 
    as unique associated local minimum (we may have to shrink $V_0$).
%    \item \textit{Comment on required neighborhoods for Theorem~\ref{th:convergence}}:
    Define $\cup_{\Bth \in V_0}(\{ \Bth \} \times U_{\Bth}(\Bom_0))=V_0 \times U_0$. 
    \item \textit{Comment on compatibility with global setting:} 
    By using a suitable regularization (e.g. weight decay) for the networks, we can assume that 
    the algorithm Eq.~\eqref{iter1} and Eq.~\eqref{iter2} 
    always stays in $V_0 \times U_0$. 
    This heuristically justifies that for $(\Bth_0, \Bom_0)$ 
    we localize to $V_0 \times U_0$.
    %and are able to apply the results from stochastic approximation theory.
    \item \textit{Comment on drawbacks of assumptions:} 
    A completely rigorous justification of this argument %that does not require the localization argument and regularization, 
    would require a more thorough analysis of SGD, 
    which would of course be a very interesting future research direction. 
    %For example, it would be very convenient 
    %to have sufficient smoothness conditions on $\tilde{f}$ and $\tilde{h}$ 
    %that guarantee the fifth bullet point of (A6), which is usually hard to verify. 
    For SGD with one time scale, a result in this direction can be found in \citet{Mertikopoulos:20}. It would be interesting to extend it to two timescales.
   % Results in this direction, which are related 
  %  to the so called lock-in probability, at least for one time-scale, 
  %  can be found e.g.\ in \citep{Karmakar:16}.
    \item \textit{Comment on requirements for critical points:} 
    It is also a widely accepted (but yet unproven) conjecture
    that the probability of ending in a poor local minimum is very small 
    for sufficiently large networks, see e.g.\ \citet{Choromanska:15,Kawaguchi:16,Kawaguchi:17,Kawaguchi:19,Kawaguchi:19nc}. 
    Thus, we can ensure that Eq.~\eqref{iter1} and Eq.~\eqref{iter2} 
    really converges to a useful quantity (a high quality local minimum), 
    if our networks are large enough. 
    %MH: include this maybe
    %It is also safe to exclude saddles, as SGD is known avoids strict saddle points/manifolds
    %ith probability 1 under weak smoothness assumptions, cf. e.g. \cite{Mertikopoulos:20}.
\end{itemize}
Using these smoothness assumptions, it is not hard to ensure the required properties in (A2) and (A6) by relating the analysis of the loss surfaces of $L_g$ and $L_h$ to a stability analysis of the corresponding gradient systems. Further technical details can be found in Section \ref{sec:dettheo1} in the appendix.

\paragraph{(A5) and (A7) Transition Kernel and Bounded Iterates.}
The transition kernel is continuous (c.f. Section \ref{sec:dettheo1} in the appendix).
Boundedness of $\Bth_n$ and $\Bom_n$ 
is achieved by weight decay terms in practice.

{\bf Proof of Theorem~\ref{th:convergence}.}
\vspace{-0.5cm}
\begin{proof}
In the previous paragraphs %(which were named according to the corresponding assumptions (A1)-(A7))
%MH: commented this out, can, however, also leave it if you want
%"Assumption (A1): Definition of the Markov Control",
%"Assumption (A3): martingale Difference Property and the Probabilistic Setting",
%"Assumption (A2) and Assumption (A6): Ensuring Smooth and Lipschitz $h$ and $f$ and Convergence of the ODEs
%via Assumptions on the Losses", and  
%"Assumption (A5) and Assumption (A7): Transition Kernel and Bounded Iterates."
we discussed how the assumptions of Theorem~\ref{th:karmakar} can be fulfilled. 
%SEPP previous list can be written more compact.
%Either by showing properties or making assumptions on our setting.
%Therefore we apply Theorem~\ref{th:karmakar} to 
%obtain the statements of Theorem~\ref{th:convergence}.
\end{proof}

\vspace{-1.0cm}

%MH: better embedding in the text, needs rewriting!
\paragraph{Finite Greediness is Sufficient to Converge to the Optimal Policy.}
\label{sec:finte_greedy}
Regularization terms are weighted by $(z_2)_n$ which converges to zero,
therefore the optimal policies are the same as without the regularization.
There exists an optimal policy $\pi^*$ that is 
deterministic according to Proposition~4.4.3 in \citet{Puterman:05}.
We want to ensure via a parameter $(z_1)_n$ that the policy becomes more 
greedy during learning. 
If the policy is not greedy enough, estimates of the action-value 
or the advantage function may misguide the algorithm and the
optimal policy is not found. For example, huge negative rewards if
not executing the optimal actions may avoid convergence to
the optimal policy if the policy is not greedy enough.
$(z_1)_n$ directly enters the policy according to Eq.~\eqref{eq:policy}.
We show that we can estimate how large $(z_1)_n$ must become in order
to ensure that $Q$-value and policy gradient methods converge to an
optimal policy, 
if it is the local minimum of the loss function (we cannot ensure this). 
For policy gradients, the optimal actions receive always the
largest gradient and the policy converges to the optimal policy. 
The required greediness will be measured by the parameter $\beta>1$.
{\em In practical applications we know that $\beta$ exists but do not
know its value, since it depends on
characteristics of the task and the optimal $Q$-values.}
For a more formal treatment c.f. Section \ref{sec:fgreedy} in the appendix, especially Lemma \ref{lem:approx}.

\paragraph{Conclusions and Outlook} 
We showed local convergence of an abstract actor-critic setting and applied it to a version of PPO and RUDDER under practical assumptions.
%We have proven the convergence of a version PPO and RUDDER under commonly used assumptions.
We intend to apply our results to similar practically relevant settings, e.g. the PPO algorithm discussed in \citet{Schulman:17}.
A further future direction is to guarantee convergence to
an optimal policy.
% Just added this, if it's ok for you.
It would also be interesting to relax some of the required assumptions on the loss functions (e.g. by extending the techniques in \citet{Mertikopoulos:20} to two timescales) or elaborate on convergence rates.

\subsubsection*{Acknowledgments}
The ELLIS Unit Linz, the LIT AI Lab, the Institute for Machine Learning, are supported by the Federal
State Upper Austria. IARAI is supported by Here Technologies. We thank the projects AI-MOTION
(LIT-2018-6-YOU-212), DeepToxGen (LIT-2017-3-YOU-003), AI-SNN (LIT-2018-6-YOU-214),
DeepFlood (LIT-2019-8-YOU-213), Medical Cognitive Computing Center (MC3), PRIMAL (FFG873979), S3AI (FFG-872172), DL for granular flow (FFG-871302), ELISE (H2020-ICT-2019-3 ID:
951847), AIDD (MSCA-ITN-2020 ID: 956832). We thank Janssen Pharmaceutica, UCB Biopharma
SRL, Merck Healthcare KGaA, Audi.JKU Deep Learning Center, TGW LOGISTICS GROUP
GMBH, Silicon Austria Labs (SAL), FILL Gesellschaft mbH, Anyline GmbH, Google Brain, ZF
Friedrichshafen AG, Robert Bosch GmbH, Software Competence Center Hagenberg GmbH, TÜV
Austria, and the NVIDIA Corporation.

\bibliography{bib}
\bibliographystyle{iclr2021_conference}

\clearpage
\appendix
\section{Appendix}

\subsection{Further Details on PPO} \label{sec:ppo}
%\lipsum[1-6]
Here we describe the minimization problem for the PPO setup in a more detailed way by including the exact expression for the gradients of the respective loss functions:
%For h: a suggestion how to write this
% Just a small issue: I think \hat(q) and \hat(v) have to depend on \omega, right?
\begin{align} 
 \label{lossh2a}
    &\rL_h(\Bth_n, \Bom_n, \Bz_n)  \ =  \  \EXP_{\tau \sim \pi(\Bth_n,\Bz_n) } 
    \left[ - \ G_0 \ + \ (z_2)_n \ \rho(\tau,\Bth_n,\Bz_n) \right]  \ , \\
     \nonumber
     &h(\Bth_n, \Bom_n, \Bz_n) \ =  
    \EXP_{\tau \sim \pi(\Bth_n,\Bz_n) } 
    \left[ - \ \sum_{t=0}^T \nabla_{\Bth} \log \pi(a_t \mid s_t ; \Bth_n,\Bz_n) \ 
    \left( \hat{q}^{\pi}(s_{t},a_{t};\Bom_n) \ - \ \hat{v}^\pi(s_t;\Bom_n) \right) \right. \\ \nonumber
     &\left. + \ (z_2)_n \  \sum_{t=0}^T \nabla_{\Bth_n} \log \pi(a_t \mid s_t ; \Bth_n,\Bz_n) 
     \ \rho(\tau,\Bth_n,\Bz_n)
     \ + \ (z_2)_n  \ \nabla_{\Bth_n}  \rho(\tau,\Bth_n,\Bz_n) \right] \ , \\
 \label{lossg2a}
 % Here we do SARSA
     &\rL^\mathrm{TD}_g(\Bth_n, \Bom_n, \Bz_n) \ = 
     \EXP_{\tau \sim \pi(\Bth_n,\Bz_n)} \left[\frac{1}{2} \
      \sum_{t=0}^{T} \big( \delta^{\mathrm{TD}}(t) \big)^2 \right]  \ , \\
    &f^\mathrm{TD}(\Bth_n, \Bom_n, \Bz_n) \ =       
    \EXP_{\tau \sim \pi(\Bth_n,\Bz_n)} \left[
     - \sum_{t=0}^{T}   \delta^{\mathrm{TD}}(t) \ \nabla_{\Bom_n}  \hat{q}^{\pi}(s_t,a_t;\Bom_n) \right]  \ , \\
 \label{lossg2MCa}
     &\rL^\mathrm{MC}_g(\Bth_n, \Bom_n, \Bz_n) \ = 
     \EXP_{\tau \sim \pi(\Bth_n,\Bz_n)} \left[\frac{1}{2} \
      \sum_{t=0}^{T} \bigg( G_t \ - \ \hat{q}^{\pi}(s_t,a_t;\Bom_n) \bigg)^2 \right]  \ , \\
      \nonumber 
    &f^\mathrm{MC}(\Bth_n, \Bom_n, \Bz_n) \ = \  \EXP_{\tau \sim \pi(\Bth_n,\Bz_n)} \left[
      -\sum_{t=0}^{T} \bigg(  G_t \ - \ \hat{q}^{\pi}(s_t,a_t;\Bom_n) \bigg) \ \nabla_{\Bom_n}  \hat{q}^{\pi}(s_t,a_t;\Bom_n) \right]  \ , \\
\label{iter1PPOa}
 &\Bth_{n+1}  \  =  \ \Bth_n  \ -  \ a(n)  \ \hat{h} (\Bth_n, \Bom_n, \Bz_n) \ , \quad
 \Bom_{n+1}  \ =  \ \Bom_n  \  - \ b(n) \ \hat{f}(\Bth_n, \Bom_n, \Bz_n)  \ ,
\end{align}

\subsection{Further details on RUDDER} \label{sec:rudder}
%\lipsum[1-6]
In a similar vein we present the minimization problem of RUDDER in more detail:
\begin{align} \label{lossh1a}
    &\rL_h(\Bth_n, \Bom_n,\Bz_n) \ = \
    \EXP_{\tau \sim \breve{\pi}}
    \left[ \frac{1}{2} \ \sum_{t=0}^{T} \bigg( R_{t+1}(\tau ; \Bom_n)  -   \hat{q}(s_t, a_t; \Bth_n)\bigg)^2 \ + \ (z_2)_n \ \rho_{\Bth}(\tau,\Bth_n,\Bz_n) \right] \\ 
    &h(\Bth_n, \Bom_n,\Bz_n) \ =  \  
    \EXP_{\tau \sim \breve{\pi}}
    \left[-\sum_{t=0}^{T} \bigg( R_{t+1}(\tau ; \Bom_n)  -   \hat{q}(s_t, a_t; \Bth_n)\bigg) \ 
    \nabla_{\Bth} \hat{q}(s_t, a_t; \Bth_n) \right. \\ \nonumber
     &\Bigg.\  \qquad + (z_2)_n \ \nabla_{\Bth} \rho_{\Bth}(\tau,\Bth_n,\Bz_n) \ \Bigg]   \\ 
      \label{lossg1a}
     &\rL_g(\Bth_n, \Bom_n, \Bz_n) \ = \
     \EXP_{\tau \sim \pi(\Bth_n,\Bz_n)} \left[\frac{1}{2} \
     \bigg( \sum_{t=0}^{T} \tilde{R}_{t+1} \ - \   g( \tau ;
     \Bom_n ) \bigg)^2  \ + \ (z_2)_n \ \rho_{\Bom}(\tau,\Bth_n,\Bz_n) \right]  \\ 
    &f(\Bth_n, \Bom_n, \Bz_n) \ = \ \EXP_{\tau \sim \pi(\Bth_n,\Bz_n)} \left[
     -\bigg( \sum_{t=0}^{T} \tilde{R}_{t+1} \ - \   g( \tau ;
     \Bom_n ) \bigg) \ \nabla_{\Bom} g( \tau ;  \Bom_n ) \  \right. \\
     & \nonumber \qquad \Bigg. +(z_2)_n \ \nabla_{\Bom} \rho_{\Bom}(\tau,\Bth_n,\Bz_n) \Bigg] \ , \\
\label{iter1RUDDERa}
 &\Bth_{n+1}  \ =  \ \Bth_n  \ -  \ a(n)  \ \hat{h} (\Bth_n, \Bom_n, \Bz_n) \ , \quad
 \Bom_{n+1}  \ =  \ \Bom_n  \  -  \ b(n) \ \hat{f}(\Bth_n, \Bom_n, \Bz_n)  \ ,
\end{align}
    
\subsection{Causality and Reward-To-Go} \label{causality}
%\lipsum[1-6]
% JAM -> We need to say that this is assuming an MDP, otherwise probabilities of actions are not independent and cannot be decomposed into a product. p_\pi(\tau) = \prod \pi(a_t \mid s_t) if we assume MDP
% MH: dont we assume this anyway, we introduced the MDP framework in main paper, or is sthg missing?
This section is meant to provide the reader with more details concerning the causality assumption that leads to the formula for $h$ in Eq.~\eqref{lossh2a} for PPO.
We can derive a
formulation of the policy gradient with reward-to-go.
For ease of notation, instead of using $\tilde{P_{\pi}}(\tau)$ as in previous sections, we here denote the probability of state-action sequence
$\tau=\tau_{0,T}=(s_0,s_0,s_1,a_1,\ldots,s_T,a_T)$
with policy $\pi$ as
\begin{align}
 &  p(\tau) \ = \ p(s_0) \ \pi(a_0 \mid s_0) \
 \prod_{t=1}^{T} p(s_t \mid s_{t-1},a_{t-1}) \  \pi(a_t \mid s_t)
 \\\nonumber
 &= \  p(s_0) \ \prod_{t=1}^{T} p(s_t \mid s_{t-1},a_{t-1}) \
  \prod_{t=0}^{T}  \pi(a_t \mid s_t) \ .
\end{align}
The probability of state-action sequence
$\tau_{0,t}=(s_0,s_0,s_1,a_1,\ldots,s_t,a_t)$
with policy $\pi$ is
\begin{align}
 &  p(\tau_{0,t}) \ = \ p(s_0) \ \pi(a_0 \mid s_0) \
 \prod_{k=1}^{t} p(s_k \mid s_{k-1},a_{k-1}) \  \pi(a_k \mid s_k)
 \\\nonumber
 &= \  p(s_0) \ \prod_{k=1}^{t} p(s_k \mid s_{k-1},a_{k-1}) \
  \prod_{k=0}^{t}  \pi(a_k \mid s_k) \ .
\end{align}
The probability of state-action sequence
$\tau_{t+1,T}=(s_{t+1},a_{t+1},\ldots,s_T,a_T)$
with policy $\pi$ given $( s_t,a_t)$ is
\begin{align}
 &  p(\tau_{t+1,T} \mid s_t,a_t) \ = \ 
 \prod_{k=t+1}^{T} p(s_k \mid s_{k-1},a_{k-1}) \  \pi(a_k \mid s_k)
 \\\nonumber
 &= \  \prod_{k=t+1}^{T} p(s_k \mid s_{k-1},a_{k-1}) \
  \prod_{k=t+1}^{T}  \pi(a_k \mid s_k) \ .
\end{align}

The expectation of $\sum_{t=0}^{T} R_{t+1}$ is
\begin{align}
  & \EXP_{\pi} \left[ \sum_{t=0}^{T} R_{t+1} \right] \ =
  \  \sum_{t=0}^{T} \EXP_{\pi} \left[ R_{t+1}  \right]\ .
\end{align}
With $R_{t+1} \sim p(r_{t+1} \mid s_t,a_t)$, the
random variable $R_{t+1}$ depends only on $(s_t,a_t)$.
We define the expected reward
$\EXP_{r_{t+1}} \left[ R_{t+1} \mid s_t,a_t\right]$
as a function $r(s_t,a_t)$ of  $(s_t,a_t)$:
\begin{align}
    r(s_t,a_t) \ &:= \ \EXP_{r_{t+1}} \left[ R_{t+1} \mid
      s_t,a_t\right] \ = \ \sum_{r_{t+1}} p(r_{t+1} \mid s_t,a_t) \
    r_{t+1}\ .
\end{align}
{\bf Causality.}
We assume that the reward
$R_{t+1}=R(s_t,a_t)  \sim p(r_{t+1} \mid s_t,a_t)$
only depends on the past but not on the future.
The state-action pair $(s_t,a_t)$ is determined by
the past and not by the future. Relevant is only
how likely we observe  $(s_t,a_t)$ and not what we
do afterwards.

Causality is derived from the Markov property of the MDP and means:
\begin{align}
  &\EXP_{\tau \sim \pi} \left[ R_{t+1}  \right]
  \ = \ \EXP_{\tau_{0,t} \sim \pi} \left[ R_{t+1}  \right] \ .
\end{align}
That is

\begin{align}
  &\EXP_{\tau \sim \pi} \left[ R_{t+1}  \right] \ = \
  \sum_{s_1} \sum_{a_1} \sum_{s_2} \sum_{a_2} \ \ldots \
  \sum_{s_T} \sum_{a_T} p(\tau) \ r(s_t,a_t)  \\ \nonumber
  &= \   \sum_{s_1} \sum_{a_1} \sum_{s_2} \sum_{a_2} \ \ldots \
  \sum_{s_T} \sum_{a_T} \ \prod_{l=1}^{T} p(s_l \mid s_{l-1},a_{l-1}) \
  \prod_{l=1}^{T}  \pi(a_l \mid s_l) \ r(s_t,a_t)\\ \nonumber
  &= \   \sum_{s_1} \sum_{a_1} \sum_{s_2} \sum_{a_2} \ \ldots \
  \sum_{s_t} \sum_{a_t} \ \prod_{l=1}^{t} p(s_{l} \mid s_{l-1},a_{l-1}) \
  \prod_{l=1}^{t}  \pi(a_{l} \mid s_{l}) \ r(s_t,a_t)\\ \nonumber
  &~~~\sum_{s_{t+1}} \sum_{a_{t+1}} \sum_{s_{t+2}} \sum_{a_{t+2}} \ \ldots \
  \sum_{s_T} \sum_{a_T} \ \prod_{l=t+1}^{T} p(s_{l} \mid s_{l-1},a_{l-1}) \
  \prod_{l=t+1}^{T}   \pi(a_{l} \mid s_{l})\\ \nonumber
  &= \   \sum_{s_1} \sum_{a_1} \sum_{s_2} \sum_{a_2} \ \ldots \
  \sum_{s_t} \sum_{a_t} \ \prod_{l=1}^{t} p(s_{l} \mid s_{l-1},a_{l-1}) \
  \prod_{l=1}^{t}  \pi(a_{l} \mid s_{l}) \  r(s_t,a_t) \\ \nonumber
  &= \   \EXP_{\tau_{0,t} \sim \pi} \left[ R_{t+1}  \right]    \ .
\end{align}

{\bf Policy Gradient Theorem.}
We now assume that the policy $\pi$ is parametrized by $\Bth$,
that is,  $\pi(a_t \mid s_t) = \pi(a_t \mid s_t ; \Bth)$.
We need the gradient with respect to $\Bth$ of
$\prod_{t=a}^{b}  \pi(a_t \mid s_t)$:
\begin{align}
  & \nabla_{\theta}  \prod_{t=a}^{b}  \pi(a_t \mid s_t ; \Bth) \  = \
  \sum_{s=a}^{b}  \prod_{t=a,t \not= s}^{b}  \pi(a_t \mid s_t ; \Bth)
  \ \nabla_{\theta}  \pi(a_s \mid s_s ; \Bth) \\ \nonumber
  &= \  \prod_{t=a}^{b}  \pi(a_t \mid s_t ; \Bth)  \
  \sum_{s=a}^{b} \frac{ \nabla_{\theta}  \pi(a_s \mid s_s ; \Bth)}
  {\pi(a_s \mid s_s ; \Bth)}\\ \nonumber
  &= \  \prod_{t=a}^{b}  \pi(a_t \mid s_t ; \Bth)  \ \sum_{s=a}^{b}
  \nabla_{\theta}  \log \pi(a_s \mid s_s ; \Bth) \ .
\end{align}
It follows that
\begin{align}
   & \nabla_{\theta}   \EXP_{\pi} \left[  R_{t+1} \right] \ = \
    \EXP_{\pi} \left[ \sum_{s=1}^{t} \nabla_{\theta}  \log \pi(a_s
      \mid s_s ; \Bth)  \  R_{t+1}  \right] \ . 
\end{align}

We only have to consider the {\bf reward to go}.
Since $a_0$ does not depend on $\pi$, we have
$\nabla_{\theta} \EXP_{\pi} \left[R_1 \right]=0$.
Therefore
\begin{align}
    &\nabla_{\theta} \EXP_{\pi} \left[ \sum_{t=0}^{T} R_{t+1} \right] \ = \
      \sum_{t=0}^{T} \nabla_{\theta} \EXP_{\pi} \left[ R_{t+1} \right] \\ \nonumber
   &= \
    \EXP_{\pi} \left[  \sum_{t=1}^{T} \sum_{k=1}^{t} \nabla_{\theta}  \log \pi(a_k
      \mid s_k ; \Bth)  \   R_{t+1}  \right] \\ \nonumber
    &= \ \EXP_{\pi} \left[  \sum_{k=1}^{T} \sum_{t=k}^{T} \nabla_{\theta}  \log \pi(a_k
      \mid s_k ; \Bth)  \   R_{t+1}  \right]\\ \nonumber
    &= \ \EXP_{\pi} \left[  \sum_{k=1}^{T} \nabla_{\theta}  \log \pi(a_k
      \mid s_k ; \Bth)  \  \sum_{t=k}^{T} R_{t+1}  \right]\\ \nonumber
    &= \ \EXP_{\pi} \left[  \sum_{k=1}^{T} \nabla_{\theta}  \log \pi(a_k
      \mid s_k ; \Bth)  \  G_k \right]   \ . 
\end{align}

We can express this by $Q$-values.
\begin{align}
   &\EXP_{\pi} \left[  \sum_{k=1}^{T} \nabla_{\theta}  \log \pi(a_k
      \mid s_k ; \Bth)  \  G_k \right]\\ \nonumber
   &= \  \sum_{k=1}^{T} \EXP_{\pi} \left[  \nabla_{\theta}  \log \pi(a_k
      \mid s_k ; \Bth)  \  G_k \right]\\ \nonumber
   &= \ \sum_{k=1}^{T} \EXP_{\tau_{0,k} \sim \pi} \left[  \nabla_{\theta}  \log \pi(a_k
      \mid s_k ; \Bth)  \ \EXP_{\tau_{k+1,T} \sim \pi} \left[ G_k \mid
     s_k,a_k \right] \right]\\ \nonumber
   &= \ \sum_{k=1}^{T} \EXP_{\tau_{0,k} \sim \pi} \left[  \nabla_{\theta}  \log \pi(a_k
      \mid s_k ; \Bth)  \ q^{\pi}(s_k,a_k)  \right]\\ \nonumber
   &= \ \EXP_{\tau \sim \pi} \left[  \sum_{k=1}^{T}  \nabla_{\theta}  \log \pi(a_k
      \mid s_k ; \Bth)  \ q^{\pi}(s_k,a_k)  \right] \ . 
\end{align}
We have finally:
\begin{align}
    &\nabla_{\theta} \EXP_{\pi} \left[ \sum_{t=0}^{T} R_{t+1} \right] \ = \
   \EXP_{\tau \sim \pi} \left[  \sum_{k=1}^{T}  \nabla_{\theta}  \log \pi(a_k
      \mid s_k ; \Bth)  \ q^{\pi}(s_k,a_k)  \right] \ . 
\end{align}

\subsection{Precise statement of Assumptions} \label{ch:precise}
%\lipsum[1-6]
Here we provide a precise formulation of the assumptions from \citet{Karmakar:17}. The formulation we use here is mostly taken from \citet{Heusel:17}:
\begin{enumerate}[wide, labelwidth=!, labelindent=0pt,label=\textbf{\large (A\arabic*)}]
\item \textit{Assumptions on the controlled Markov processes:} The controlled Markov process
$\Bz$ takes values in a compact metric space $S$.
It is controlled by the iterate sequences $\Bth_n\}$
and  $\Bom_n$ and furthermore $\Bz_n$ by a random process $\Ba_n$ taking values in a
compact metric space $W$.
For $B$ Borel in $S$ the $\Bz_n$ dynamics for $n\geq0$ is determined by a transition kernel $\tilde{p}$:
\begin{align*}
&\rP(\Bz_{n+1} \in B |\Bz_l,
\Ba_l, \Bth_l, \Bom_l, l\leq n) 
= \
\int_{B} \tilde{p}(\Rd \Bz| \Bz_n,
\Ba_n, \Bth_n, \Bom_n).
\end{align*}

\item \textit{Assumptions on the update functions:}
$h :  \dR^{m+k} \times S^{(1)} \to \dR^m$ is  
jointly continuous as well as Lipschitz in 
its first two arguments, and uniformly w.r.t.\ the third. 
This means that for all $ \Bz \in S$:
\begin{align*}
\|h(\Bth, \Bom, \Bz) \ - \ h(\Bth', \Bw',
  \Bz)\| \leq \ L^{(1)} \ (\|\Bth-\Bth'\| + \|\Bom -
  \Bom'\|) \ .
\end{align*}

Similarly for $f$, where the Lipschitz constant is $L^{(2)}$.

\item \textit{Assumptions on the additive noise:}
For $i=1,2$, $\{(\Bm_i)_n\}$ are martingale difference sequences with bounded second moments.
More precisely,
$(\Bm_i)_n$ are martingale difference sequences
w.r.t.\ increasing $\sigma$-fields
\[
\mathfrak{F}_n \ = \ \sigma(\Bth_l, \Bom_l, (\Bm_1)_{l}, (\Bm_2)_{l},
          \Bz_l, l \leq n) ,
\]
satisfying 
$
\rE \left[\|(\Bm_i)_n \|^2 \mid \mathfrak{F}_n \right]
  \ \leq \ B_i
$
for $n \geq 0$ and a given constants $B_i$.

\item 
\textit{Assumptions on the learning rates:} \label{asslr}

\begin{align*}
&\sum_{n} a(n) \ = \ \infty, \quad 
\sum_{n} a^2(n) \ < \ \infty,  \\
&\sum_{n} b(n) \ = \ \infty, \quad 
\sum_{n} b^2(n) \ < \ \infty, \\
\end{align*}
and $a(n) \ = \ \Ro(b(n))$. Furthermore, $a(n), b(n), 
n \geq 0$ are non-increasing. 

\item \textit{Assumptions on the transition kernels:}
The state-action map

\begin{align*}
S \times W \times \dR^{m+k} \ni &(\Bz,\Ba,\Bth,\Bom)  
\mapsto \ \tilde{p}(\Rd\By \mid \Bz, \Ba,
  \Bth, \Bom)
\end{align*}

is continuous (the topology on the spaces of probability measures is induced by weak convergence).

\item \textit{Assumptions on the associated ODEs:}
We consider occupation measures
which intuitively give for 
the controlled Markov process the probability or 
density to observe a particular
state-action pair from $S \times W$ 
for given $\Bth$ and $\Bom$ and a given control. %policy $\pi$. 
A precise definition of these occupation measures can be found e.g.\ 
on page~68 of \citet{Borkar:08} or page~5 in \citet{Karmakar:17}.
We have following assumptions:
\begin{itemize}[topsep=-5pt,itemsep=1pt,partopsep=0pt, parsep=0pt,wide, labelwidth=!, labelindent=0pt]
\item We assume that there exists only one such ergodic occupation measure for $\Bz_n$ on $S \times W$, denoted by $\Gamma_{\Bth,\Bom}$.
%for the prescribed $\Bth$ and $\Bom$ 
A main reason for assuming uniqueness is that it enables us to deal with ODEs instead of differential inclusions.
Moreover, set 
$\tilde{f}(\Bth, \Bom) \ = \ \int f(\Bth,\Bom,\Bz) \ \Gamma_{\Bth,\Bom}(\Rd\Bz, W)$.
\item We assume that for $ \Bth \in \dR^m$, the ODE
$
\dot{\Bom}(t) \ = \ \tilde{f}(\Bth,\Bom(t)) 
$
has a unique asymptotically stable equilibrium $\Bla(\Bth)$ with attractor set $B_{\Bth}$ 
such that $\Bla :  \dR^m \to \dR^k$ is a Lipschitz map with global Lipschitz constant.
\item The Lyapunov function $V(\Bth,.)$ associated to $\Bla(\Bth)$ is continuously differentiable.
\item Next define
$
\tilde{h}(\Bth) \ = \ \int h(\Bth,\Bla(\Bth),\Bz) \ \Gamma_{\Bth,\Bla(\Bth)}(\Rd\Bz, W).
$
We assume that the ODE
$
\dot{\Bth}(t) \ = \ \tilde{h}(\Bth(t)) 
$
has a global attractor set $A$.
\item For all $\Bth$, with probability 1, $\Bom_n$ for $n\geq 1$ 
belongs to a compact subset 
$Q_{\Bth}$
of $B_{\Bth}$ ``eventually''.
\end{itemize}
\vspace{1mm}
This assumption is an adapted version of (A6)' of \citet{Karmakar:17}, 
to avoid too many technicalities (e.g. in \citet{Karmakar:17} two controls are used, which we avoid here to not overload notation).

\item \textit{Assumption of bounded iterates:}
$\sup_n \| \Bth_n \| \ < \ \infty$ and $\sup_n \| \Bom_n \| \ < \ \infty$ a.s.

\end{enumerate}

\subsection{Further Details concerning the Sampling Process} \label{sec:sampling}
%\lipsum[1-6]
Let us formulate the construction of the sampling process in more detail: We introduced the function $S_{\pi}$ in the main paper as follows:
\[
S_{\pi}: \Omega \to \tilde{\Omega}_{\pi},\ x \mapsto \argmax_{\tau \in \tilde{\Omega}_{\pi}} \left\{ \sum_{\eta\le \tau} \tilde{P_{\pi}}(\eta) \le x \right\}.
\]
%We already introduced an ordering  $\le_{\pi}$ on $\tilde{\mathcal{A}_{\pi}}$ according to the magnitude of the values of its events on $P_{\pi}$.
Now $S_{\pi}$ basically divides the interval $[0,1]$ into finitely many disjoint subintervals, such that the $i$-th subinterval $I_i$ maps to the $i$-th element $\tau_i \in\tilde{\Omega}_{\pi}$, 
%if we assume the ordering $\le_{\pi}$ described above, 
and additionally the length of $I_i$ is given by $\tilde{P_{\pi}}(\tau_i)$. $S_{\pi}$ is measurable, because the pre-image of any element of the sigma-algebra $\tilde{\mathfrak{A}_{\pi}}$ wrt. $S_{\pi}$ is just a finite union of subintervals of $[0,1]$, which is clearly contained in the Borel-algebra. Basically $S_{\pi}$ just describes how to get one sample from a multinomial distribution with (finitely many) probabilities $\tilde{P_{\pi}}(\tau)$, where $\tau \in \tilde{\Omega}_{\pi}$. Compare with inverse transform sampling, e.g. Theorem 2.1.10. in \citet{Casella:02} and applications thereof. For the reader's convenience let's briefly recall this important concept here in a formal way:
\begin{lemma}[Inverse transform sampling]
Let $X$ have continuous cumulative distribution $F_X(x)$ and define the random variable $Y$ as $Y=F_{X}(X)$. Then $Y$ is uniformly distributed on $(0,1)$.
\end{lemma}

\subsection{Further Details for Proof of Theorem 1} \label{sec:dettheo1}
%\lipsum[1-6]
Here we provide further technical details needed to ensure the assumptions stated before to prove our main theorem Theorem 1. 
\begin{enumerate}[wide, labelwidth=!, labelindent=0pt,label=\textbf{\large (A\arabic*)}]
\item \textit{Assumptions on the controlled Markov processes:}  
%$z_n$ clearly is a time-homogeneous Markov process with $\delta_{\beta}$ as its unique invariant measure. So integrating wrt.\ this invariant measure will in our case just correspond to obtaining the policy $\pi^+$.
Let us start by discussing more details for controlled processes that appear in the PPO and RUDDER setting. Let us focus on the process related to $(z_1)_n$:
Let $\beta>1$ and let the real sequence $z_n$ be defined by $(z_1)_1=1$ and $(z_1)_{n+1}=(1-\frac{1}{\beta})(z_1)_{n}+1$.
The $z_n$'s are nothing more but the partial sums of a geometric series converging to $\beta$.

The sequence $(z_1)_n$ can also be interpreted as a time-homogeneous Markov process $(\Bz_1)_n$ with transition probabilities given by
\begin{equation} \label{def:transprop}
P(z, y)=\delta_{(1-\frac{1}{\beta})z+1},
\end{equation}

where $\delta$ denotes the Dirac measure, and with the compact interval $[1,\beta]$ as its range. We use the standard notation for discrete time Markov processes, described in detail e.g. in \citet{Hairer:18}. Its unique invariant measure is clearly $\delta_{\beta}$.
So integrating wrt.\ this invariant measure will in our case just correspond to taking the limit $(z_1)_n \to \beta$.

\item \textit{$h$ and $f$ are Lipschitz:} By the mean value theorem it is enough to show that the derivatives wrt.\ $\Bth$ and $\Bom$ are bounded uniformly wrt.\  $\Bz$. We only show details for $f$, since for $h$ similar considerations apply. By the explicit formula for $L_g$, we see that $f(\Bth,\Bom,\Bz)$ can be written as:
\begin{align*}
 \sum_{\substack{s_1,..,s_T \\a_1,...,a_T}}  &\prod_{t=1}^{T} p(s_t \mid s_{t-1},a_{t-1}) \pi(a_t \mid s_t, \Bth,\Bz) 
  \nabla_{\Bom}\Phi(g(\tau ; \Bom, \Bz) ,\tau, \Bth, \Bom, \Bz) . 
\end{align*}
The claim can now be readily deduced from the assumptions (L1)--(L3).

\item \textit{Martingale difference property and estimates:} From the results in the main paper on the probabilistic setting, 
$(\Bm_1)_{n+1}$ and $(\Bm_2)_{n+1}$ can easily be seen to be martingale difference sequences wrt. their filtrations $\mathfrak{F}_n$. Indeed, the sigma algebras created by $\Bom_n$ and $\Bth_n$ already describe $\tilde{\mathfrak{A}}_{\pi_{\Bth_n}}$, and thus:
\[
\EXP[(\Bm_i)_{n+1}|\mathfrak{F}_n]=\EXP[\hat{f}(\Bth_n, \Bom_n, \Bz_n)|\mathfrak{F}_n]-\EXP[f(\Bth_n, \Bom_n, \Bz_n)]=0.
\] It remains to show that
$
\EXP[|(\Bm_i)_{n+1}|^2 | \mathfrak{F}_n] \le B_i
$ for $i=1,2.$
This, however, is also clear, since all the involved expressions are bounded uniformly again by the assumptions (L1)--(L3) on the losses (e.g. one can observe this by writing down the involved expressions explicitly as indicated in the previous point (A2) ).

\item \textit{Assumptions on the learning rates:} These standard assumptions are taken for granted.

\item \textit{Transition kernels:} The continuity of the transition kernels is clear from Eq.~\eqref{def:transprop} (continuity is wrt. to the weak topology in the space of probability measures. So in our case, this again boils down to using continuity of the test functions).

\item \textit{Stability properties of the ODEs:} 
\begin{itemize}[topsep=-5pt,itemsep=1pt,partopsep=0pt, parsep=0pt,wide, labelwidth=!, labelindent=0pt]
\item By the explanations for (A1) we mentioned that integrating wrt. the ergodic occupation measure in our case corresponds to taking the limit $\Bz_n \to \Bz$ (since our Markov processes can be interpreted as sequences).
Thus $\tilde{f}(\Bth, \Bom)=f(\Bth,\Bom,\Bz)$. In the sequel we will also use the following abbreviations:
%MH: check consistency of notation for limit of z_n
$f(\Bth,\Bom)=f(\Bth,\Bom,\Bz)$, $h(\Bth,\Bom)=h(\Bth,\Bom,\Bz)$, etc.. Now consider the ODE
\begin{equation} \label{fastode}
  \dot{\Bom}(t)=f(\Bth,\Bom(t)),  
\end{equation}
where $\Bth$ is fixed. Eq.~\eqref{fastode} can be seen as a gradient system for the function $L_g$. By  standard results on gradient systems (cf. e.g. Section 4 in \citet{Absil:06} for a nice summary), which guarantee equivalence between strict local minima of the loss function and asymptotically stable points of the associated gradient system, we can use the assumptions (L1)--(L3) and the remarks thereafter from the main paper to ensure that there exists a unique asymptotically stable equilibrium $\Bla(\Bth)$ of Eq.~\eqref{fastode}.
\item The fact that $\Bla(\Bth)$ is smooth enough can be deduced by the Implicit Function Theorem as discussed in the main paper. 
\item For Eq.~\eqref{fastode} $L_g(\Bth,\Bom)-L_g(\Bth,\Bla(\Bth))$ can be taken as associated Lyapunov function $V_{\Bth}(\Bom)$, and thus $V_{\Bth}(\Bom)$ clearly is differentiable wrt.\ $\Bom$ for any $\Bth$. 
\item The slow ODE
$
  \dot{\Bth}(t)=h(\Bth(t),\Bla(\Bth(t))  
$
also has a unique asymptotically stable fixed point, which again is guaranteed by our assumptions and the standard results on gradient systems. 
\end{itemize}
\item \textit{Assumption of bounded iterates:} This follows from the assumptions on the loss functions.
\end{enumerate}

\subsection{Finite Greediness is sufficient to converge to the optimal policy} \label{sec:fgreedy}
%\lipsum[1-6]

Here we provide details on how the optimal policy can be deduced using only a finite parameter $\beta>1$.
The $Q$-values for policy $\pi$ are:
\begin{align*}
    &q^{\pi}(s_t,a_t) \ = \ \EXP_{\pi}
    \left[  \sum_{\tau=t}^{T} R_{\tau+1} \mid s_t,a_t \right] \ = \ \sum_{\substack{s_t,..,s_T \\a_t,...,a_T}}
   \prod_{\tau=t}^{T-1} p(s_{\tau+1} \mid s_{\tau},a_{\tau}) \
  \prod_{\tau=t}^{T}  \pi(a_{\tau} \mid s_{\tau}) \ \sum_{\tau=t}^{T} R_{\tau+1}  \ .
\end{align*}
The optimal policy $\pi^*$ is known to be 
deterministic $\left(\prod_{t=1}^T \pi^*(a_t\ |\ s_t) \in \{0,1\} \right)$.
Let us assume that the optimal policy is also unique.
Then we are going to show the following result:
\begin{lemma}
\label{lem:approx}

For  $i_{\max}= \arg\max_{i} q^{\pi^*}(s,a^i)$ and $v^{\pi^*}(s) = \max_{i} q^{\pi^*}(s,a^i)$.
We define
\begin{align} 
\label{eq:defeps}
     0 \ &< \  \epsilon \ < \ \min_{s,i\not=i_{\max}} (v^{\pi^*}(s) \  - \ q^{\pi^*}(s,a^i))  \ ,
\end{align}
We assume a function $\psi(s,a^i)$ that defines the actual policy $\pi$ via
\begin{align} 
 \pi(a^i  \mid s; \beta) \ &= \ \frac{\exp(\beta \ \psi(s,a^i) ) }
  {\sum_j \exp(\beta \ \psi(s,a^j) )} \ .
\end{align}
We assume that the function $\psi$ already identified the optimal actions,
which will occur during learning at some time point when the policy is getting more greedy:
\begin{align} 
     0 \ &< \ \delta \ < \ \min_{s,i\not=i_{\max}} (\psi(s,a^{i_{\max}}) \  - \psi(s,a^i) )  \ .
\end{align}
Hence,
\begin{align} 
 \lim_{\beta \to \infty} \pi(a^i  \mid s; \beta) \ &= \ \pi^*(a^i  \mid s) \ .
\end{align}

We assume that
\begin{align} \label{eq:betadef}
 \beta \ > \max \left( \frac{\log(\ABS{\sA}-1)}{\delta}, 
 -\log \left( \frac{\epsilon}{2 T \ (\left| \sA \right| - 1) \ |\sS|^T \ |\sA|^T \ (T+1) \ K_R} \right) /  \delta \ \right) .
\end{align}
%MH: Somehow we should state that \beta should be greater than 1 for technical reasons. However not relevant in this lemma.

Then we can make the statement for all $s$:
\begin{align} 
  \forall_{j,j \not=i}: \ q^{\pi}(s,a^i) \ &> \ q^{\pi}(s,a^j) \ \Rightarrow \ i = i_{\max} \ ,
\end{align}
therefore the $Q$-values $q^{\pi}(s,a^i)$ determine the optimal policy as the action with
the largest $Q$-value can be chosen.

More importantly, $\beta$ is large enough to allow 
$Q$-value based methods and policy gradients converge to the 
optimal policy if it is the local minimum of the loss functions.  
For $Q$-value based methods the optimal action can be 
determined if the optimal policy is the minimum of the loss functions. 
For policy gradients the optimal action receives always the
largest gradient and the policy converges to the optimal policy.

\end{lemma}
\begin{proof}
We already discussed that the optimal policy $\pi^*$ is known to be 
deterministic $\left(\prod_{t=1}^T \pi^*(a_t\ |\ s_t) \in \{0,1\} \right)$.
Let us assume that the optimal policy is also unique.
Since
\begin{align} 
 \pi(a^i  \mid s; \beta) \ &= \ \frac{\exp(\beta \ (\psi(s,a^i) \ - \ \psi(s,a^{i_{\max}})) ) }
  {\sum_j \exp(\beta \ \psi(s,a^j)  \ - \ \psi(s,a^{i_{\max}}))} \ ,
\end{align}
we have
\begin{align} 
 \pi(a^{i_{\max}}  \mid s; \beta) \ &= \ \frac{1}
  {1 \ + \ \sum_{j,j\not=i_{\max}}  \exp(\beta \ \psi(s,a^j)  \ - \ \psi(s,a^{i_{\max}}))}
  \ > 
  \frac{1}
  {1 \ + \ (|\sA|-1) \ \exp(- \ \beta \ \delta) } \\ \nonumber &= \
  1 \ - \ \frac{(|\sA|-1) \ \exp(- \ \beta \ \delta)}{1 \ + \ (|\sA|-1) \ \exp(- \beta \ \delta) }
  \ > 
  1 \ - \ (|\sA|-1) \ \exp(- \ \beta \ \delta) 
\end{align}
and for $i \not= i_{\max}$
\begin{align} 
 \pi(a^i  \mid s; \beta) \ &= \ \frac{\exp(\beta \ (\psi(s,a^i) \ - \ \psi(s,a^{i_{\max}})) ) }
  {1 \ + \ \sum_{j,j\not=i_{\max}}  \exp(\beta \ \psi(s,a^j)  \ - \ \psi(s,a^{i_{\max}}))}
  \ < \  \exp(- \ \beta \ \delta) \ .
\end{align}

For $\prod_{t=1}^{T}  \pi^*(a_t \mid s_t) = 1$, we have
\begin{align}
  \prod_{t=1}^{T}  \pi(a_t \mid s_t) \ &> \ (1 \ - \ (|\sA|-1) \ \exp(- \ \beta \ \delta))^T
  \ > \ 1 - \  T \ (|\sA|-1) \ \exp(- \ \beta \ \delta)\ ,
\end{align}
where in the last step we used that $(|\sA|-1) \exp(- \beta \delta)<1$ by definition of $\beta$ in \eqref{eq:betadef} so that an application of Bernoulli's inequality is justified.
For $\prod_{t=1}^{T}  \pi^*(a_t \mid s_t) = 0$, we have
\begin{align}
  \prod_{t=1}^{T}  \pi(a_t \mid s_t) \ &< \ \exp(- \ \beta \ \delta) \ .
\end{align}
Therefore
\begin{align} \label{eq:policy_prod}
 \ABS {\prod_{t=1}^{T}  \pi^*(a_t \mid s_t) \ - \ \prod_{t=1}^{T}  \pi(a_t \mid s_t)}
  \ &< \ T \ (|\sA|-1) \ \exp(- \ \beta \ \delta) \ .
\end{align}
Using Eq.~\eqref{eq:policy_prod} and the definition of $\beta$ in Eq.~\eqref{eq:betadef} we get:
\begin{align}
  & \ABS {q^{\pi^*}(s,a^i) \ - \ q^{\pi}(s,a^i)}   \\\nonumber  
  &= \ \left|
     \sum_{\substack{s_1,..,s_T \\a_1,...,a_T}} \ \prod_{t=1}^{T} p(s_t \mid s_{t-1},a_{t-1}) \
  \left( \prod_{t=1}^{T}  \pi^*(a_t \mid s_t) \ - \ \prod_{t=1}^{T}  \pi(a_t \mid s_t) \right) \ \sum_{t=0}^{T} R_{t+1} \right| \\ \nonumber
  & < \ 
     \sum_{\substack{s_1,..,s_T \\a_1,...,a_T}} \ \prod_{t=1}^{T} p(s_t \mid s_{t-1},a_{t-1}) \
  \left| \prod_{t=1}^{T}  \pi^*(a_t \mid s_t) \ - \ \prod_{t=1}^{T}  \pi(a_t \mid s_t) \right| \ (T+1) \ K_R\\ \nonumber
  & < \ 
     \sum_{\substack{s_1,..,s_T \\a_1,...,a_T}}  \ \left| \prod_{t=1}^{T}  \pi^*(a_t \mid s_t) \ - \ \prod_{t=1}^{T}  \pi(a_t \mid s_t) \right| \ (T+1) \ K_R\\ \nonumber
  & < \ 
     |\sS|^T \ |\sA|^T \ \frac{\epsilon}{ 2|\sS|^T \ |\sA|^T \ (T+1) \ K_R}  \ (T+1) \ K_R  \ = \ \epsilon / 2 \ .
\end{align}

%Using $\beta = $,
%now we can bound 
%\begin{align}
% \ABS {q^{\pi^*}(s,a^i) \ - \ q^{\pi}(s,a^i)} \ &< \ \epsilon / 2 \ . 
%\end{align}
Now from the condition that $q^{\pi}(s,a^i) \ > \ q^{\pi}(s,a^j) $ for all $j \ne i$ we can conclude that 
\begin{align}
\ q^{\pi^*}(s,a^j) \ - \ q^{\pi^*}(s,a^i)  
  \ < \ (q^{\pi}(s,a^j) \ + \ \epsilon / 2) \  - \ (q^{\pi}(s,a^i) \ - \ \epsilon / 2)
   \ < \  \epsilon \ 
\end{align}
for all $j \ne i$. Thus for $j \not=i$ it follows that $j \not=i_{\max}$ and
consequently $i=i_{\max}$.

\end{proof}

%\bibliography{bib}
%\bibliographystyle{iclr2021_conference}

%%%%%%%%%%%%%%%%%%%%%%%%%%%%%%%%%%%%%%%%%%%%%%%%%%%%%%%%%%%%%%%%%%%%%%%%%%%%%%%
%%%%%%%%%%%%%%%%%%%%%%%%%%%%%%%%%%%%%%%%%%%%%%%%%%%%%%%%%%%%%%%%%%%%%%%%%%%%%%%

\end{document}